\documentclass{article}
\usepackage{geometry}
\makeatletter
\def\@biblabel#1{[#1]}
\makeatother
\usepackage{amsmath,amsfonts,amssymb}
\usepackage{thm-restate}
\usepackage{hyperref}
\usepackage{tikz}
\usepackage{url}
\usetikzlibrary{arrows,automata,patterns}
\usetikzlibrary{calc}
\usetikzlibrary{decorations.pathreplacing,calligraphy}

\newtheorem{theorem}{Theorem}[section]
\newtheorem{lemma}[theorem]{Lemma}
\newtheorem{definition}[theorem]{Definition}

\newtheorem{example}[theorem]{Example}

\usepackage[noend]{algorithm2e}
\usepackage{stmaryrd}
\usepackage{graphicx}
\usepackage{microtype}
\usepackage{wrapfig}
\usepackage{todonotes}

\usepackage{macros}
\newcommand{\myorcid}[1]{\unskip\texorpdfstring{%
\href{https://orcid.org/#1}{\includegraphics[width=10px]{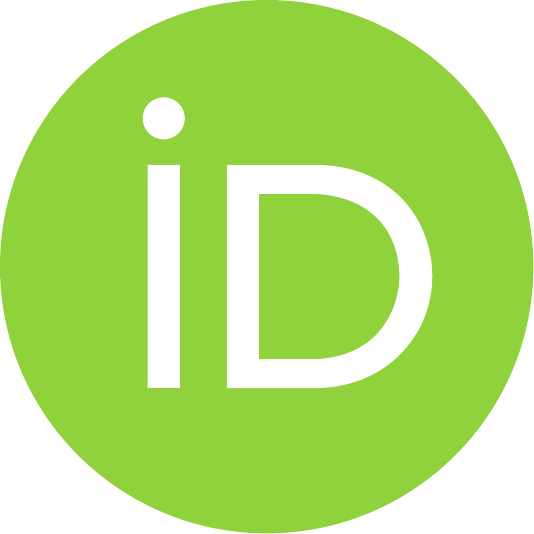}}}{}}

\begin{document}

% \title{Online Test Synthesis From Requirements:\\ Enhancing Reinforcement Learning \\ with Game Theory}
\title{Requirement-Based Testing: Enhancing Reinforcement Learning with Game Theory}

% \author{
%  \IEEEauthorblockN{Ocan Sankur\IEEEauthorrefmark{1}\orcid{0000-0001-8146-4429}, 
%  Thierry J\'eron\IEEEauthorrefmark{1}\orcid{0000-0002-9922-6186},
%  Nicolas~Markey\IEEEauthorrefmark{1}\orcid{0000-0003-1977-7525},
%  David Mentr\'e\IEEEauthorrefmark{2}\orcid{0000-0003-4315-0335},
%  Reiya Noguchi\IEEEauthorrefmark{3}}
%  \IEEEauthorblockA{\IEEEauthorrefmark{1}Univ Rennes, Inria, CNRS, Irisa, Rennes, France}
%  \IEEEauthorblockA{\IEEEauthorrefmark{2}Mitsubishi Electric R\&D Centre Europe, Rennes, France}
%  \IEEEauthorblockA{\IEEEauthorrefmark{3}Mitsubishi Electric Corporation, Tokyo, Japan}
% }
\author{
  Ocan~Sankur\textsuperscript{1},
  Thierry~J\'eron\textsuperscript{1},
  Nicolas~Markey\textsuperscript{1},
  David~Mentr\'e\textsuperscript{2},
  Reiya~Noguchi\textsuperscript{3}\\[0.5em]
  \small\textsuperscript{1}Univ Rennes, Inria, CNRS, Rennes, France\\
  \small\textsuperscript{2}Mitsubishi Electric R\&D Centre Europe, Rennes, France\\
  \small\textsuperscript{3}Mitsubishi Electric Corporation, Tokyo, Japan
}
\maketitle
%----------------------------------------------------------------------
\begin{abstract}
  We~consider the automatic online synthesis of black-box test cases from functional requirements specified as automata for reactive implementations.    
  The goal of the tester is to reach some given state, to satisfy a coverage criterion,
  while monitoring the violation of the requirements.
  We develop an approach based on Monte Carlo Tree Search (MCTS), which is a classical
  technique in reinforcement learning for computing optimal policies. 
  Seeing the automata requirements as a game between the implementation and the tester,
  we develop a heuristic by biasing the search towards inputs that are promising in this game.
  We experimentally show that our heuristic accelerates the convergence of MCTS,
  thus improving the performance of testing.
\end{abstract}
%----------------------------------------------------------------------

\section{Introduction}
Requirement engineering and testing are two important and related phases in the development process.
Indeed, test cases are usually derived from functional specifications and documented with the requirements they are supposed to check.
Several existing tools allow one to automatically generate test cases from formal functional specifications  and/or formal requirements (see~\eg, the tools~\href{https://www.t-vec.com/}{T-VEC}~\cite{blackburn1996t}, \href{https://torxakis.org/}{TorXakis}~\cite{tretmans2019model} or
\href{https://www.3ds.com/products/catia/stimulus}{Stimulus}~\cite{jeannet2016debugging}, and  surveys~\cite{utting2012taxonomy} and~\cite{li2018survey}).%\todo{TJ: Mention ``requirement-based test generation'' ? Restrict to reactive systems ? T-VEC, Stimulus ? }
%This helps the development process since the developer can
%focus on writing formal specifications or requirements, and test cases are generated automatically.

We are interested in black-box testing reactive implementations w.r.t. automata requirements.
Reactive implementations are programs that alternate between reading an input valuation and writing an output valuation,
which are typical, e.g. of controllers for manufacturing systems~\cite{bolton2015programmable}.
Requirements are given as automata recognizing valid sequences of valuations of input and output variables.
The~\emph{conformance} of an implementation to a set of requirements is formalized as the absence of input-output valuation sequences generated by the implementation
that are rejected by the automaton.
Black-box testing means that the tester does not have access to the source code,
and only interacts with the system via feeding inputs and observing outputs~\cite{beizer1995black}.

The goal of testing is to drive the implementation to some particular state or show some particular behavior where non-conformance 
is suspected. These are described by \emph{test objectives}.
They are typically derived from requirements or from coverage criteria, \eg~state or transition coverage. 
We give several examples throughout the paper.

There are several black-box testing algorithms and tools in the literature.
The~closest to our approach are the tools TorX and TorXakis~\cite{tretmans2003torx,tretmans2019model,torxakis} which are based on the {\bf ioco} testing theory~\cite{Tretmans96-SCT}. These allow the user to specify the automata-based requirements reading input-output valuations as well as test objectives (called test purposes) in a process algebra language. 
They perform tests by picking random input valuations, and observing the outputs from the implementation, while checking for non-conformance.
Because tests are performed using random walks in this approach, deep traces satisfying the test objective or violating the requirements are hard to find in practice.

% \paragraph{Reinforcement Learning for Testing.}
This issue has been addressed in many works by interpreting the test synthesis in a \emph{reinforcement learning}(RL) setting~\cite{sutton2018reinforcement}.
Reinforcement learning is a set of techniques for computing strategies that optimize a given reward function based on interactions of an agent with its environment. 
%It~has been applied to learn strategies for playing board games such as Chess and~Go~\cite{silver2016mastering}. 
Here testing is seen as a game between the tester and the implementation: the former player selects inputs, and the latter player selects outputs. 
Because the implementation is black-box, the tester
is playing an unknown game, but can discover the game through interactions.
Using a game approach for test synthesis has long been advocated~\cite{Yannakakis2004};
and online testing for interface automata were considered in \cite{VRC-fates2006} using RL. 

Because reaching a test objective is a 0/1 problem (an execution either reaches the objective and has a reward~1, or~does not reach~it and has reward~0),
RL~algorithms are usually very slow in finding deep traces.
The application of~RL to black-box testing thus requires the use of \emph{reward shaping}~\cite{ng1999policy}
which consists in assigning intermediate rewards to steps before the objective is reached; these are used to guide the search to more promising 
input sequences and can empirically accelerate convergence.
Reward shaping has been used for testing, \eg, in~\cite{DBLP:journals/stvr/KorogluS21} where an RL algorithm (Q-learning) was used for testing GUI applications with respect to linear temporal logic (LTL) specifications;
rewards were then computed based on transformations on the target LTL formula.
A similar approach was applied for testing distributed algorithm implementations in \cite{BEMN-oopsla24}
by manually specifying \emph{waypoints} which are steps the test algorithm must reach to reach the objective.
We focus here on improving RL-based test algorithms w.r.t. automata-based requirements, which, combined with reward shaping and the heuristic we develop here,
as a systematic way of obtaining such waypoints \emph{automatically} to guide the search.

\smallskip
\noindent\textit{Contributions}
% Although reinforcement learning helps one to guide the search towards the test objective,
% these methods can still be slow in finding traces satisfying the test objective,
% especially when the number of input bits is high, and when the traces to be found are long.
Our target is to improve the performance of black-box online testing based on reinforcement learning.
More precisely, we develop heuristics for a Monte Carlo Tree Search (MCTS)
algorithm applied in this setting, based on a game-theoretic analysis
of the requirement automaton, combined with an appropriate reward shaping scheme.

Monte-Carlo Tree Search~\cite{coulom2006efficient} (see also \cite{browne2012mcts}) is a RL technique to compute optimal policies in games.
It~consists in exploring the available moves
randomly, while estimating the average reward of each newly-explored move
and updating the reward estimates of previously selected moves.
More precisely, it uses a so-called \emph{tree policy} to select which branch of the tree to explore, and a \emph{roll-out policy} to evaluate the reward of a branch.

Our main contribution is a heuristic for biasing both the tree policy and the roll-out policy in MCTS in order to
reach test objectives faster, while maintaining convergence guarantees in the limit.
The heuristic is based on a \emph{greedy} test strategy we introduce, defined as follows.
We first see the requirement automaton as a turn-based game between the tester (selecting inputs) and the implementation (selecting outputs).
We compute the winning region using standard algorithms~\cite{fijalkow2025games}: inside this region, the tester can ensure reaching the test objective. Outside of this region, we compute \emph{cooperative} strategies which guarantee reaching the test objective provided that the implementation selects a particular output in one step. We alternate the computation of winning regions and cooperative states to cover the full reachable (and coreachable) state space.
Our game-based greedy test strategy consists in selecting uniformly at random inputs that are part of a winning strategy if any, or allow the implementation to cooperate.

Our novel testing algorithm consists in integrating the greedy strategy inside the tree and roll-out policies.
The probability of using the greedy strategy vanishes as time progresses, so the algorithm is equivalent to the standard MCTS in the limit;
inheriting its convergence properties.

We~present experimental results on two distributed algorithms and an implementation of a planning problem
to show the effectiveness of our approach in comparison with a basic MCTS solution and uniform random testing.

% \todo[inline]{Check  and its references.
% - Without requirements. Combinatorial testing for sequential programs that read a single input valuation: 
% produce inputs in which every combination of parameters is covered by at least once \cite{NXSD-2005}. QuickCheck~\cite{claessen2000quickcheck}.
% These do not apply to reactive systems which are in continuous interaction with the tester,

%%%%%%%%%%%%%%%%%%%%%%%%%%%%%%%%%%%%%%%%%%%%%%%%%%%%%%%%%%%%%%%%%%%%%%%%%%%
\section{Preliminaries}
%%%%%%%%%%%%%%%%%%%%%%%%%%%%%%%%%%%%%%%%%%%%%%%%%%%%%%%%%%%%%%%%%%%%%%%%%%%
\label{section:preliminaries}
We fix a set of atomic propositions~$\AP$, partitioned into
$\AP^{\inp} \uplus \AP^{\out}$, that represent Boolean input- and output
variables of the system.
A~\emph{valuation} of~$\AP$ is an element $\vAP$ of~$2^\AP$ determining the set of atomic propositions which are true
(or~equivalently, it~is a mapping~$\vAP\colon \AP \rightarrow\{\top,\bot\}$).  We~denote
by~$\vAP^\inp$ (respectively~$\vAP^\out$) the~projection of~$\vAP$
on~$\AP^{\inp}$ (resp.~$\AP^{\out}$) such that $\vAP= \vAP^{\inp} \uplus
\vAP^{\out}$.  We~write~$\BC(\AP)$ for the set of Boolean combinations
of atomic propositions in~$\AP$.  That a valuation~$\vAP$ satisfies a
formula~$\phi\in\BC(\AP)$, denoted by~$\vAP\models \phi$, is~defined
in the usual~way.

In considered reactive systems, at~each~step, the~environment first sets an input
valuation~$\vAP^{\inp}$, then the system immediately sets an output
valuation~$\vAP^{\out}$.
The~valuation observed at this step is
thus~$\vAP = \vAP^{\inp} \uplus \vAP^{\out}$.
A~\emph{trace} of the system is a sequence $\sigma= {\vAP}_1\cdot{\vAP}_2\cdots {\vAP}_n$ of input and output valuations.

% Internal variables may be used by the system to compute outputs from
% the inputs and the internal state, but these are not observable to the outside.

\paragraph{Automata}
We use automata to express requirements, and as models for the
implementations under test.
%We consider that requirements of such reactive systems are given by
%automata recognizing traces.
When considered as requirements, 
they monitor the system
through the observation of the values of the Boolean variables.
Transitions of automata are guarded with Boolean constraints on~$\AP$
that need to be satisfied for the automaton to take that   transition. For
convenience, we~handle input- and output valuations separately. 
Formally,
\begin{definition}\label{def-automata}
  An \emph{automaton} is a tuple $\calA=\tuple{S=S^\inp \uplus S^\out,
    s_\init, \AP, T, F}$ where $S$~is a finite set of states, $s_\init
  \in S^\inp$~is the initial state, $T \subseteq (S^\inp \times
  \BC(\AP^\inp) \times S^\out) \uplus (S^\out \times \BC(\AP^\out)
  \times S^\inp)$ is a finite set of transitions, and $F \subseteq S$
  is the set of accepting states.
\end{definition}
%For a transition~$t=(s,g,s')\in T$ of an automaton, we~define its
%source ${\src(t)=s}$, its~target ${\tgt(t)=s'}$, and its~guard
%${\guard(t)=g}$.\todo{NM: I think we don't use this. TJ: probably not}
%

For two states~$s^\inp$ and~$s'^\inp$ and a valuation~$\vAP$, 
we write $s^\inp \fleche{\vAP} s'^\inp$ when there exist
%some valuation~$\vAP$ and
a state~$s^{\out}$ and transitions ${(s^\inp,g^\inp,s^\out)}$ and
$(s^\out, g^\out,s'^\inp)$ in~$T$ such that ${\vAP^\inp \models
  g^\inp}$ and ${\vAP^\out \models g^\out}$.

For~a~trace
$\sigma={\vAP}_1\cdot{\vAP}_2\cdots {\vAP}_n$ in $(2^\AP)^*$, we~write
$s^\inp \fleche{\sigma} s'^\inp$ if there are states
$s^\inp_0,s^\inp_1,\ldots s^\inp_n$ such that $s^\inp_0=s^\inp$,
$s^\inp_n=s'^\inp$, and for all $i\in [1,n]$, $s^\inp_{i-1}
\fleche{{\vAP}_i} s^\inp_i$.  A~trace~$\sigma \in (2^\AP)^*$ is
accepted by~$\calA$ if
%there exists a run of~$\calA$
%over~$\sigma$ that ends in~$F$.
$s_\init \fleche{\sigma} s$ for some~$s\in F$.
We~denote by~$\Tr(\calA)$ the set of
accepted traces.

An automaton is \emph{input-complete} if from any
(reachable) state~$s^\inp$
%(resp.~$s^\out$)
and any valuation~$\vAP^\inp \in 2^{\AP^\inp}$,
%(resp.~$\vAP^\out \in 2^{\AP^\out}$),
there is a transition~$(s^\inp,g^\inp,s'^\out)$
%(resp.~$(s^\out,g^\out,s'^\inp)$)
in~$T$ such that $\vAP^\inp\models g^\inp$.
%(resp.~$\vAP^\out\models g^\out$).
%\todo{TJ: I dont think output completion is necessary}
It~is \emph{output-complete} if a similar requirement holds for
states in~$S^\out$ and valuations in~$\vAP^{\out}$, and it is \emph{complete}
if it is both input- and output-complete. 
%This~is no
%loss of generality since missing transitions can be directed to some
%non-accepting trap state.  We~also assume our automata to be
An~automaton is 
\emph{deterministic} when, for~any two transitions $(s,g_1,s_1)$ and
$(s,g_2,s_2)$ issued from a same source~$s$, if $g_1\et g_2$ is
satisfiable, then $s_1=s_2$.

\begin{figure}
  \centering
  \begin{tikzpicture}[xscale=1,yscale=0.8]
    \draw (0,0) node[rond4] (s0in) {} node {$s_0$}; 
    \draw[latex'-] (s0in.135) -- +(135:3mm);
    \draw (2,0) node[rond4] (s1in) {$s_1$}; 
    \draw (4,0) node[rond4,double] (s2in) {} node {$o$}; 
    \draw (1.2,0.8) node[rond4] (sink) {} node {$t$};
    \draw (1,0) node[carre3] (s0a) {};
    \draw (3,1) node[carre3] (s1b) {};
    \draw (3,-1) node[carre3] (s1c) {};
    \draw (1.7,-1.2) node[carre3] (s1a) {};
    \draw (-0.8,-1.2) node[carre3] (s0bc) {};
    \begin{scope}[-latex']
      \draw
      (s0in) edge[-latex'] node[below] {$a$} (s0a)
      (s0in) edge[-latex',bend left] node[below] {$b,c$} (s0bc)
      (s1in) edge[bend left,-latex'] node[below] {$a$} (s1a)
      (s1in) edge[-latex'] node[below] {$b$} (s1b)
      (s1in) edge[-latex'] node[below] {$c$} (s1c)
      (s0a) edge[-latex'] node[above] {$0$} (s1in)
      (s1a) edge[-latex',bend left] node[left] {$0,1$} (s1in)
      (s0bc) edge[-latex',bend left] node[left] {$0,1$} (s0in)
      (s1b) edge[-latex',out=180,in=90] node[below,pos=0.8] {$0$} (s0in)
      (s1c) edge[-latex',out=-100,in=-90] node[below] {$0$} (s0in)
      (s1b) edge[-latex'] node[above] {$1$} (s2in)
      (s1c) edge[-latex'] node[above] {$1$} (s2in)
      (s0a) edge[-latex'] node[left,pos=0.7] {$1$} (sink);
    \end{scope}  
  \end{tikzpicture}
  \caption{Example of a (deterministic) automaton expressing requirements.}\label{fig-exaut}
\end{figure}

We are mainly interested in $S^\inp$,
since states in $S^\out$ are intermediary states that help us distinguish input
and output valuations.
%When dealing with deterministic
%automata, we will write $(s^\inp,\vAP^\inp)$ for the
%unique state~$s^\out$ of~$S^\out$ for which there is a transition
%$(s^\inp, g^\inp, s^\out)$ such that $\vAP^\inp\models g^\inp$.
%
For a state~$s^\inp$ and a valuation~$\vAP$, we~let
$\Post_{\calA}(s^\inp,\vAP)$ denote the set of states~$s'^\inp$ such
that $s^\inp \xrightarrow{\vAP} s'^\inp$, and
$\Pre_{\calA}(s^\inp,\vAP)$ denote the set of states~$s'^\inp$ such
that $s'^\inp \fleche{\vAP} s^\inp$. Notice that for deterministic
complete automata, $\Post_{\calA}(s^\inp,\vAP)$ is a singleton.
%\todo{TJ: if also output-complete (see remark above)}

The~set of \emph{immediate predecessors} $\Pre_\calA(B)$
of a set $B\subseteq S^{\inp}$,
and the set of its \emph{immediate successors}
$\Post_\calA(B)$ are defined respectively as
$\Pre_\calA(B) = \bigcup_{s^\inp \in B,  \vAP \in 2^\AP} \Pre_\calA(s^\inp,\vAP)$,
and $\Post_\calA(B)= \bigcup_{s^\inp \in B,  \vAP \in 2^\AP} \Post_\calA(s^\inp,\vAP)$.

From these sets, one~can define the set of states from which $B$~is
reachable (\ie,~that are co-reachable from~$B$), as $\Pre_\calA^*(B)=
\lfp(\lambda X.(B  \, \cup \, \Pre_\calA(X)))$, and the set of states that are
reachable from $B$, $\Post_\calA^*(B)= \lfp(\lambda X.(B \, \cup \, \Post_\calA(X)))$,
where~$\lfp$ denotes the least-fixpoint operator.

Safety automata form a subclass of automata having  a
distinguished set~$\error\subseteq S^\inp$ of error states that are
non-accepting and absorbing (\ie,~no~transitions leave~$\error$), and
complement the set~$F$ of accepting states in~$S$
(\ie,~$F=S \setminus\error$). %\todo{TJ: output-completion will go to Error states ? Perhaps add this information here}
% Those automata describe safety properties: nothing bad happened as long as $\error$ is not reached.

The product of two automata is defined by combining pairs of transitions and 
defining the guard of the new transition as the conjunction of the guards of the two transitions (See Appendix~\ref{app:definitions}).

For the sake of readability, in Figure~\ref{fig-exaut}, we use input- and output letters
instead of atomic propositions. Here,
$\{a,b\}$ are input letters, and $\{0,1\}$ are output letters.
This automaton is deterministic;
moreover, setting $\error=\{t\}$ makes it a safety automaton. It~could
be made complete by adding looping input-output transitions (similar
to the transitions to the bottom left of~$s_0$) on~$t$ and~$o$.  %   \todo{TJ: the automaton should also be complete in trap states}
% \end{example}

\section{Testing from Requirements}
\label{sec:input}
We want to use testing to check whether a system implementation satisfies
its requirements.  We thus formalize a notion of conformance to a set
of requirements.

We~identify a requirement with its complete deterministic safety automaton, and write~$R$ for~both.
For~a~set~$\calR= \{R_i\}_{i\in  I}$ of requirements, each~specified by an automaton~$R_i$, we~denote
by $\calAR$ the product automaton ${\otimes_{i\in I} R_i}$.

% \begin{definition}
  For any requirement $R$ defined by a complete deterministic automaton, and any finite
  trace~$\sigma$, we~write
  $\sigma \fails  R$ if running~$\sigma$ in~$R$ from its initial state enters its error set~$\error_R$.
% \end{definition}
For a trace~$\sigma$ and a \emph{set} of requirements~$\calR$ %$\{R_i\}_{i \in I}$,
we~write $\sigma \fails \calR$ %(resp $\sigma \succ \calR$)  
to mean that~$\sigma \fails  \calAR$. %  (resp. $\sigma \succ \otimes \calR$).
% {\color{blue}
% Note the following simple facts, consequence of the definition of $\error$ states in the product:
% if  ${\sigma \fails \calR}$ then  ${\sigma \fails R}$ for at least one $R$ in $\calR$; 
% given $\calR' \subseteq \calR$,
% for~any trace~$\sigma$,
% if~${\sigma \fails \calR'}$, then ${\sigma \fails \calR}$.
% }

We want to test a system against a set of requirements~$\calR$.
We~consider a deterministic system implementation~$I$
(the~implementation under~test), producing Boolean traces
in~$(2^\AP)^*$.
We~assume that this system is a black box that proceeds as follows:
at~each~step, an~input valuation in~$2^{\AP^\inp}$
is provided to the system by the tester, and the system answers by producing an
output valuation in~$2^{\AP^\out}$ (on~which the tester has no
control).  
We consider deterministic implementations with the following assumptions.
$I$~behaves as an unknown finite automaton over $\AP$;
it is \emph{input-complete}, meaning that any valuation in $2^{\AP^\inp}$ 
can be set by the tester,
%\footnote{This is not restrictive since an implementation
%refusing some input valuations can be simulated by an input-complete
%implementation that would set a dedicated output variable to true in
%case of input refusal.};
and from any input state of~$I$, it is possible to reset~$I$ to the unique initial state
at any time.
Furthermore, $I$ is \emph{output-deterministic}, \textit{i.e.} any state in~$S^\out$ has exactly one
transition. This ensures that if one feeds the implementation with an input sequence 
from the initial state, then the system produces a unique output sequence.
We~denote by~$\calI$ the class of all such implementations producing traces in~$(2^\AP)^*$.
The~behaviour of~$I$ is characterized by the set of all traces produced by the interaction between the tester and the system.
Denote by~$\Tr(I)$ the set of traces produced by~$I$.

We say that a system implementation $I \in \calI$ \emph{conforms} to a set of
requirements~$\calR$ on~$\AP$ if for all $\sigma \in \Tr(I)$,
it is not the case that $\sigma \fails \calR$.

Appendix~\ref{app:examples} contains examples of requirements and corresponding automata.

\paragraph{Test Objectives.}
In testing practice, each test case is related to a particular goal, e.g., derived from a coverage criterion. 
In our case, test objectives are given by requirement automata (as in \cite{jard2005tgv}),
and the goal is to cover the accepting states.
% We formalize this now. 

% \begin{definition}
  Formally, consider a set $\calR$ of requirements, and let  $S_\calR$ denote the state space of~$\calAR$.
  A~\emph{test objective} is a set of states $O \subseteq S_\calAR^\inp$. 
%  \todo{TJ: We could generalize to a state set, but then, is covering each objective enough ?   I suppressed the case where $o \in Error$ since it is problematic later.  In the meantime, this was useful for exhaustiveness...} 
  A~trace~$\sigma$ \emph{covers}
  a test objective~$O$, if~the unique execution of $\calAR$ on $\sigma$ ends in~$O$.
%  The~trace~$\sigma$ is said to cover~$s$.
% \end{definition}
%We will consider a test objective set $\calO$ which is a set of test objectives for $\calR$.

The problem that we address in the rest of the paper is the following:
% \begin{definition}{The Test Problem}
\begin{description}
  \item[Input:]  
    requirement set~$\calR$, test objective~$O$, and deterministic implementation~$I$;
  \item[Output:]
    ~a trace of~$I$ that covers~$O$, if such a trace exists.
\end{description}
  An algorithm that solves this problem is called a \emph{test algorithm}.
A test algorithm is \emph{complete} if for any input $\calR, O$, and $I$ that contains a trace that covers~$O$,
it returns a trace covering~$O$. 
It is \emph{almost-surely complete} if in such a case, it returns a trace covering~$O$ with probability 1.

The covering traces we are looking for are thus traces of $I$ that satisfy a given objective. 
These witness the fact that we have met a particular coverage criterion.
While our testers focus on satisfying a given objective, they always monitor for errors:
if the trace fails $\calR$ then we stop and report this as an error.
One can of course define $\error$ as a test objective; in this case, the testing process
will focus on finding errors.

\subsection{Baseline Test Algorithms}
%%%%%%%%%%%%%%%%%%%%%%%%%%%%%%%%%%%%%%%%%%%%%%%%%%
Consider a set $\calR$ of requirements, 
specified by a deterministic complete automaton $\calAR=\tuple{S_\calR, \sinit^\calR, AP, T_\calR, F_\calR}$,
and an implementation modeled by an input-complete, output-deterministic
finite automaton $I=\tuple{S_I, s_\init^I, AP, T_I, F_I}$.
Recall that in our setting, the set $\calR$ of requirements, thus also its automaton $\calAR$, is known,
while the implementation~$I$ is considered to be black box to the tester.

Consider a particular test objective $O \subseteq S_\calAR^\inp$.
Let $\coreach(\calAR,O)=\Pre_{\calAR}^*(O)$ denote the set of input states of $\calAR$ from which $O$ is reachable.

Our aim is to design {online testing %\todo{NM: right?} Yes
algorithms} that compute inputs to be given to
the implementation~$I$ in order to generate a trace that either
covers~$O$, or detects non-conformance by reaching an $\error$ state (or both if $O$ contains $\error$ states);
notice that since $I$ is output-deterministic, each such input sequence
defines a unique  trace of~$I$.
The~testing
process runs as follows: from a state~$s^\calAR$ of~$\calAR$ and a
state~$s^I$ of~$I$, {the~test algorithm} %\todo{which strategy ?} 
returns an input
valuation~$\vAP^\inp$; this input valuation is fed to the
implementation, which returns an output valuation~$\vAP^\out$ and moves
to a new state~$t^I$; the~resulting valuation~$\vAP^\inp\cup\vAP^\out$
moves the automaton~$\calAR$ from~$s^\calAR$ to a new state~$t^\calAR$.
%applied from~$s^\calAR$ returns a new state~$t^\calAR$ of~$\calAR$. 
The
process then continues from~$t^\calAR$ and~$t^I$, until reaching a test objective or $\error$.

We~write $\calAR\otimes I$ for the
synchronized product of~$\calAR$ and~$I$: this is a deterministic
automaton, of which we observe only the first component (\textit{i.e.} the state of $\calAR$), while we have no information and no observation
concerning the second component except from the produced outputs.
Our~aim is to build a tester to
cover some objectives in this \emph{partially-observable} deterministic
automaton.
An example of the product $\calAR\otimes I$ built for the automaton of Fig.~\ref{fig-exaut} is given in Appendix~\ref{app:examples}.

Since we do not know~$I$, each input valuation~$\vAP^\inp$
should be selected only based on the trace generated so~far, and
possibly based on information collected on previous attempts. 
%\todo{TJ: Introduce product first, before explaining the testing process ?} 

Before explaining our test algorithms, we introduce some
vocabulary to describe the configuration where the testing process
ends.  Assume that we have generated a trace~$\sigma$ by interacting
with~$I$ from its initial state.
% Let $(s_\sigma^\calR, s_\sigma^I)$ denote the unique state of $\calAR \otimes I$ reached from the initial state
% $(\sinit^\calR,\sinit^I)$ by reading $\sigma$.
Let~$s_\sigma$ denote the state of~$\calAR$ reached after
reading~$\sigma$ from the initial state.
Four cases may occur:
if $s_\sigma \in O$, then $\sigma$ is a \emph{covering trace} for~$O$;
if $s_\sigma \in \error_{\calAR}$, then $\sigma$ is an \emph{error trace}; 
if $s_\sigma \not \in \coreach(\calAR,O)$, then $\sigma$ is \emph{inconclusive};
otherwise,  $s_\sigma \in \coreach(\calAR,O) \setminus O$ and $\sigma$ is \emph{active}.

% \begin{itemize}
%   \item if $s_\sigma \in O$, then $\sigma$ is a \emph{covering trace} for~$O$;
%   \item if $s_\sigma \in \error_{\calAR}$, then $\sigma$ is an \emph{error trace}; 
%   \item if $s_\sigma \not \in \coreach(\calAR,O)$, then $\sigma$ is \emph{inconclusive};
%   \item otherwise,  $s_\sigma \in \coreach(\calAR,O) \setminus O$ and $\sigma$ is \emph{active}.
% \end{itemize}

Intuitively, in the first case, we have found the desired covering
trace, and we~can stop.  In~the second case, we~have found a trace failing
one of the requirements of~$\calAR$, and we can also stop: the
implementation does not conform to~$\calAR$.
Notice that these two cases are not exclusive since we can have $O \cap  \error_{\calAR}\neq \emptyset$.
In~the third case,
no~matter how we extend~$\sigma$, we~will never cover~$O$; so the
tester should stop and start again by
resetting~$I$. In~the last case, $\sigma$~is active in the sense that
it might still be possible to try to extend~$\sigma$ to reach the
objective. 

It should be clear that $\sigma$ being active (last case) does not
mean that $O$ is reachable from the corresponding state~$(s_\sigma,s^I)$ of the product~$\calAR\otimes I$, 
as this depends on the (unknown) implementation~$I$ being considered:
states in~$\coreach(\calAR,O)$ are those for which \emph{some}
implementation in $\calI$ can reach~$O$. 
We illustrate this scenario in Example~\ref{example:different_I} in the appendix.

\subsubsection{Naive Uniform Testing}
\label{section:uniform}
%%%%%%%%%%%%%%%%%%%%%%%%%%%%%%%%%%%%%%%%%%%%%%%%%%
The simplest test algorithm implemented in tools such as TorXakis~\cite{torxakis}
consists in picking inputs at random at each step.
We define here an improved variant.
Let $\coreachinp(\calAR,O)$ denote the set of pairs $(s,\vAP^\inp)$
where $s$ is a state of $\calAR$, and $\vAP^\inp$ is an input valuation
for which there exists some output valuation $\vAP^\out$ such that
$\Post_{\calAR}(s,\vAP) \in \coreach(\calAR,O)$, where $\vAP = \vAP^\inp
\cup \vAP^\out$. In~fact, after observing trace~$\sigma$ ending in
state~$s$ of~$\calAR$, the~tester has no reason to give an input
$\vAP^\inp$ such that $(s,\vAP^\inp) \not \in \coreachinp(\calAR,O)$:
such an input would lead to a trace that is inconclusive, with no hope of reaching $O$.

A very simple test algorithm is the following: starting from the
initial state of the implementation, we~store in~$s$ the initial state
of~$\calAR$.  As~long as the trace being produced is active, we~select
uniformly at random an input valuation among $\{\vAP^\inp \mid
(s,\vAP^\inp) \in \coreachinp(\calAR,O)\}$.  We~observe the output
$\vAP^\out$ given by~$I$, and update~$s$ as $\Post_{\calAR}(s,
\vAP^\inp \cup \vAP^\out)$.  There are three cases when this process
stops:
% \begin{itemize}
% \item 
if~$s \in O$, then we stop and report the generated trace as a
  covering trace for~$O$;
% \item
 if $s\in \error_{\calAR}$, we~also stop and report a failure;
% \item
 if the current trace is inconclusive, then we reset~$I$, set $s$
  to the initial state of $\calAR$, and start again.
% \end{itemize}
% Note that it is possible to generate inconclusive traces
% since $\calAR$ is not assumed to be output-deterministic.
% It is then possible to have $(s,\nu^\inp) \in \coreachinp(\calAR,O)$ and for some
% $\vAP^\out$ and~$\vAP'^\out$, ${\Post_{\calAR}(s, \vAP^\inp \cup \vAP^\out)} \in
% \coreachinp(\calAR,O)$ and ${\Post_{\calAR}(s, \vAP^\inp \cup \vAP'^\out)}
% \not \in \coreachinp(\calAR,O)$ (see~e.g., Fig.~\ref{fig-exaut}).
% Some conformant implementation~$I$ can indeed return $\vAP'^\out$, producing an inconclusive trace. 
%\todo{ref to above (announced) figure}

Let this test algorithm be called $\uniformtc$. For any bound $K$, let
$\uniformtc_K$ be the uniform testing algorithm in which we stop each run after
$K$ steps, so that the generated traces have length at most~$K$.
\begin{lemma}\label{lemma-complete-uniform}
  For each requirement set~$\calR$, implementation~$I$, and test objective~$O$,
  there exists~$K>0$ such that
  if $O$ is reachable in $\calAR \otimes I$, then $\uniformtc_K$ finds a covering trace with probability~$1$.
\end{lemma}
Note that there is no need to fix~$K$. 
If $K$ is increased towards infinity, then a covering trace is found with probability 1.
Furthermore, the uniform distribution can also be relaxed: any algorithm that picks each input
of $\{\vAP^\inp \mid
(s,\vAP^\inp) \in \coreachinp(\calAR,O)\}$ with probability at least some fixed $\epsilon>0$
has this property.

\subsubsection{Testing Based on Reinforcement Learning}\label{sec-mcts}
The online-testing problem can be seen as a reinforcement-learning~(RL) problem as follows.
The considered system is the implementation~$I$, seen as a one-player deterministic game.
The goal is to find a sequence of inputs that guides the system to a given objective.
%We~assign a reward to each trace: a~covering trace has reward~$1$, other traces have reward~$0$.
We~assign a reward to each trace: a~covering trace has reward~$0$, other traces have reward~$1$.
% Note that we will consider \emph{minimizing} the reward for reasons that will be clear later.

% Notice that we do not assign a particular reward to error traces. In~fact, we~assume that the goal of the tester
% is to produce a covering trace, while monitoring all traces seen on the way for~$\calR$.
% If~an~error trace is seen, then we simply report~it. %\todo{TJ: again, I removed the case where $o$ is an error state} 
% Furthermore, it is possible to choose an objective in $\error_{\calAR}$ in which case the test strategy will try to reach an error state.

Reinforcement learning is a set of techniques that can be used to
learn strategies that maximize reward / minimize cost in
games~\cite{sutton2018reinforcement}. In this paper, we use Monte
Carlo Tree Search~\cite{coulom2006efficient,kocsis2006bandit}.
This is an algorithm explores available moves
randomly, while estimating the potential of each newly-explored move
and updating the potential of previously selected moves.

%Vanilla version.
More precisely, MCTS builds a weighted tree of possible plays of the game
iteratively as follows:
\begin{itemize}
\item {\bf Selection:} from the root of the tree, select moves, using a \emph{tree policy},
  until reaching a node where some move has not been
  explored; %\todo{TJ: rather than ``optimal'' we could say that we follow some  ``tree policy''}
\item {\bf Expansion:} add a new child corresponding to that move;
\item {\bf Simulation:} simulate a random play, using a \emph{roll-out policy}, from that new child; %\todo{we can also introduce ``rollout policy'' here}
\item {\bf Propagation:} assign the reward of that play to the new child, and
  update the rewards of its ancestors accordingly.
\end{itemize}

Different {tree policies} can be used to pick the successive
moves during the selection phase, based on statistics obtained from
previous iterations. We use UCT (Upper Confidence bounds applied to Trees)~\cite{kocsis2006bandit}
which is standard in many applications.
% At a given node of the tree, if $n$ denotes the number of total visits to this node, and $n_i$ the number of times the $i$-th child node is visited (corresponding to the $i$-th move from the parent node), and $r_i$ the current average reward of the $i$-th child, we define the score of the $i$-th child as $r_i + c\sqrt{{\ln(n)}/{n_i}}$ for some constant~$c$.
% The UCT policy consists in choosing the child with the best score.
% Intuitively, this score is the average reward $r_i$ biased in order to make sure that each child node is visited frequently enough. In fact, the second term 
% of the score is only relevant when $n_i$ is small.
% If the goal is to maximize the average score, then $c>0$, and the UCT policy picks the child node with the maximal score; if, as in this work, we want to minimize the reward, then one chooses $c<0$ and the policy picks the child node with minimal score.

Once an unvisited action has been selected and
the tree has been expanded with a new node, a~{roll-out policy} is
applied to evaluate the potential of that new action, usually by
randomly selecting inputs that form a path from the resulting configuration. This
evaluation gives a first reward to the newly created node, which is
back-propagated to all its predecessors in the tree; %\todo{repetition} 
each node of the
tree stores statistics from previous rounds, including the number of
visits and its average reward.

In the limit, the procedure is guaranteed to provide the optimal expected reward values from each state and move.
In practice, the procedure can be interrupted at any time (depending on the
available resources devoted to the search), and the current best moves
from all states of the tree provide a strategy.

In our case, each simulation is bounded by~$K$ steps.
Such a bound is necessary since some simulations might never reach the objective, an error, or an inconclusive state and thus never terminate.
We use $\uniformtc_K$ (Section~\ref{section:uniform}) as the roll-out policy. 
Choosing the inputs uniformly in the simulation phase is standard.
Here, we simply improve this by restricting the inputs to remain in the coreachable set. 

\paragraph{Reward Shaping: Accelerating Convergence}
\label{section:reward-shaping}
One technique that is used to help reinforcement-learning algorithms converge faster is \emph{reward shaping}~\cite{ng1999policy},
which consists in assigning real-valued rewards to traces, to give more information than just the binary 0/1.
For instance, if~the~trace induced a run in~$\calAR$ that became very close to the objective, then it~might be given a better (lower) reward
than another trace whose run was very~far.
The~computation of such rewards based on automata objectives were formalized in~\cite{camacho2019ltl}.
%We~now describe how we apply this to our setting.

Here $\calAR$ is used solely to compute rewards of traces, while the actual testing will be done by the MCTS algorithm.
Let $C_0=O$, and for $i\geq 1$, define 
\({C_i = \Pre_{\calAR}(C_{i-1}) {\setminus}(C_0{\cup}{\ldots}{\cup} C_{i-1})}.\)
We have $\bigcup_{i\geq 0} C_i = \coreach(\calAR, O)$
and each $C_i$ is the set of states that are at distance $i$ from some state in~$O$.
% (\textit{i.e.} $\calAR$ contains a run of length~$i$ to~$O$).

We consider two ways of assigning rewards to simulation traces. Let $m$ be maximal such that $C_m\neq \emptyset$.
Let us define $C_{m+1}=S_{\calAR}\setminus\coreach(\calAR,O)$, that is, all states that are not coreachable.
We assign each trace $\sigma$ whose run in $\calAR$ ends in $s^\calR$ the reward $\lastreward(\sigma) = k$ if, and only~if,  $s^\calR \in C_k$.
Notice that this is well defined because the sets~$C_i$ are pairwise-disjoint and they cover all states.
Hence, the closer the trace to objective~$O$, the smaller its reward. A reward of 0 means that the state satisfies the objective.

The second reward assignment considers not only the last reward, but all rewards seen during the simulation, as follows.
Let $r_0,r_1,r_2, \ldots,r_{K-1}$ denote the sequence of rewards encountered
during simulation (these are the rewards of the prefixes of the trace $\sigma$).
If the length of the simulation was less than $K$, we simply repeat the last reward to extend this sequence
to size $K$. Then the reward of the simulation is given by
\(
  \discountedreward_\gamma(\sigma) = r_{K-1} \cdot \sum_{i=0}^{K-1} \gamma^i\cdot r_i,
\)
where $\gamma\in(0,1)$ is a discount factor.
Notice that this value is 0, and minimal, if and only if $r_{K-1}=0$.
Furthermore, while $r_{K-1}$ is the most important factor, the second factor means that we favor simulations
whose first reward values are smaller. This can in fact be seen as a weighted version of $\lastreward$, where the weight is smaller
if the the first steps have small rewards.

\paragraph{Basic MCTS Testing Algorithm}
This yields the second testing algorithm we consider which we call \emph{basic MCTS}.
The algorithm is complete in the following sense since MCTS with UCT ensures that each node and action in the tree will be picked infinitely often
in the limit.
\begin{lemma} The basic MCTS algorithm is \emph{complete}:  
  For each requirement set~$\calR$, implementation~$I$, test objective~$O$,
  there exists~$K{>}0$ such that,
  if $O$ is reachable in $\calAR {\otimes} I$, then basic MCTS
  with simulation bound~$K$ finds a covering trace.
  %\todo{TJ: should we formalize ``complete'' ? Here we say ``eventually'', in Lemma 1 we say ``with proba 1'', in 5.4 we simply say ``complete'' with no precision}
\end{lemma}
\smallskip

This algorithm is not just almost-surely complete, but also complete.
This is because the UCT policy deterministically guarantees that all nodes of the tree are visited infinitely often. Thus, when the depth of the tree becomes large enough, any covering trace will be part of it, thus will have been executed.
However, we do rely on estimated rewards to guide the search to ensure faster termination in practice.
The above lemma holds for both reward assignments $\lastreward$ and $\discountedreward_\gamma$.
Moreover, as for $\uniformtc$, it is possible not to fix~$K$, but increase it slowly towards infinity.

This basic MCTS algorithm is also a baseline since it can be obtained by combining known results from the literature
following \textit{e.g.} \cite{VRC-fates2006,DBLP:journals/stvr/KorogluS21}.

\subsection{Greedy Strategies and Improved Test Algorithms}
\label{section:greedy}
We now describe our new test algorithms.
We adopt a game-theoretic view of testing in the sense of games played on graphs~\cite{fijalkow2025games}, and 
define particular strategies for the tester, then show how these can improve the basic MCTS approach.

Finding a trace of~$\calAR$ that reaches~$O$ can be seen as a turn-based game played in the automaton $\calAR$ between two players: the \emph{tester}, and the \emph{system}.
At each step, the tester provides an input valuation, and the system responds with an output valuation, and the game moves to a new state in $\calAR$.
In this game, the online testing algorithm defines the strategy used by the tester, while the system plays a fixed strategy determined by the implementation,
which is however black-box, thus unknown to the tester.

% In the preceding section we saw that the distance to the objective $O$ can be used to compute rewards in MCTS algorithms in order to accelerate convergence. 
% In the testing context, controlability is a major concern.
% The implementation $I$ itself is controlable (the output is unique for a given input), but it is unknown.
% The requirement automaton is known, but unfortunately not controlable.
% Indeed it is output-complete, all outputs being possible for each input, some possibly leading to $O$, with the cooperation of the system, some others not.
% This controlability and cooperation can be gathered in a strategy.
% We explain how MCTS can be improved by the use of this strategy, both for reward shaping and to refine the tree policy.\todo{TJ: revoir cette intro ?} 

% Because the implementation is assumed to be a black box, we~cannot know
% in advance which outputs it will produce when given a sequence of inputs.

% We~rely on game theory to model the interactions between the tester and the %unknown
% implementation.
% Howeve, since the implementation is unknown, but requirements constitute the reference of correct behaviours, the game we define below and the computed strategies are defined from the requirement automaton~$\calAR$.

\paragraph{Controllable Predecessor and Successor Operators}
Suppose we are given requirements specified as the complete deterministic safety automaton
$\calAR=\tuple{S_\calAR, \sinit^\calAR, AP, T_\calAR, F_\calAR}$
%\todo{do we need all subscripts? OS: we can remove them but this takes effort :)}
%a deterministic safety automaton $\calA=\tuple{S, s_\init, AP, T, F}$
with $\error_\calAR = S_\calAR\setminus F_\calAR$, and 
a test objective ~$O \subseteq S_\calAR^\inp$. 
We~consider the game~$(\calAR, O \cup \error_{\calAR})$
played between the~\emph{tester} providing inputs, and the \emph{system} providing outputs.
The~objective of the tester is to reach~$O$ or to reveal a non-conformance (that is, reach $\error_{\calAR}$); it~controls the input
valuations in~$2^{{\AP^\inp}}$, and observes the output valuations
in~$2^{{\AP^\out}}$ chosen by the system.

A basic notion in the study of turn-based games is that of \emph{controllable predecessors}.
The \emph{immediate controllable predecessors} of a set~$B\subseteq S^\inp_\calAR$ is the set
% Because part of the variables are uncontrollable to the tester
% (namely, the outputs, which are set by the system under
% test), we~define \emph{controllable} predecessors.  The
% \emph{immediate controllable predecessors} of a set~$B\subseteq S^\inp$
% is the set
\[
  \CPre_\calAR(B){=}\{s {\in} S^\inp \mid  \exists {\vAP}^\inp {\in} 2^{\AP^{\inp}}, \forall {\vAP}^\out {\in} 2^{\AP^{\out}},
  \Post_{\calAR}(s,{\vAP}^\inp{\cup} {\vAP}^\out) {\in} B{\cup} \error_\calAR \}.
\]

Thus, from each state~${s \in \CPre_\calAR(B)}$, the~tester can select some valuation~${\vAP}^\inp$ such that
whatever the output $\vAP^{\out}$ returned by the system, we~are guaranteed to either enter a state in~$B$, or exhibit a non-conformance.
We~say that the tester \emph{can ensure} entering~$B$ or reveal a non-conformance in~one~step (regardless of the system's strategy).

We define the set of \emph{controllable predecessors} of a subset $B$ as 
$\CPre_\calAR^*(B)= \lfp(\lambda X.(B  \cup \CPre_\calAR(X))$.
%\todo{do we need ``$\cup X$'' in fixpoint?}
Using the same
reasoning as previously, this least fixpoint defines the set of all states
from which the tester can ensure either entering~$B$ within a finite number of
steps, or reveal a non-conformance, regardless of the system's strategy.

% \begin{example}
For example, in Fig.~\ref{figW_i}, we have $\CPre_\calAR(O)=\{s_0, s'_0 \}$ and \linebreak $\CPre^*_\calAR(O)=\{s_0, s'_0 \} \cup O$.
% \end{example}

A~\emph{strategy} in this game is a function $\st\colon \Tr(\calAR) \to
2^{2^{{\AP^\inp}}}$ that associates each trace to a subset of input
valuations, those that can be applied at the next step after this trace.
A~strategy~$\st$ is said \emph{memoryless} whenever, for any two
traces~$\sigma$ and~$\sigma'$ reaching the same state in $\calAR$, it~holds $\st(\sigma)=\st(\sigma')$.

% For two  strategies $\st$ and $\st'$ , we~say that
% $\st'$ is \emph{more } than~$\st$, denoted $\st'
% \supseteq \st$, if for any trace~$\sigma$ in~$\Tr(\calA_g)$,
% $\st'(\sigma) \supseteq \st(\sigma)$.

A~trace $\sigma=\vAP_1\cdot\vAP_2\cdots\vAP_n$ is \emph{compatible} with a  strategy~$\st$ from a state~$s$ if 
there exists a sequence of states $q_0,q_1,q_2,\ldots,q_n$ such that ${q_0 = s}$ and  for all~$0\leq i < n$,
$q_i\xrightarrow{\vAP_{i+1}} q_{i+1}$, and $\vAP_{i+1}^{\inp} \in \st(\vAP_1\cdots\vAP_i)$.
We~write \linebreak $\Outcome(\st,s)$ to denote the set of all traces compatible with~$\st$ from~$s$,
also called the \emph{outcomes} of~$\st$ from~$s$.
% Clearly enough, if $\st' \supseteq \st$,  then $\Outcome(\st',s) \supseteq  \Outcome(\st,s)$ for all~$s$.

\paragraph{Winning and Cooperative Strategies}
For a given set $B\subseteq S^\inp$, a strategy~$\st$ for the tester is \emph{$B$-winning} from state~$s$
if all its outcomes from~$s$ eventually reach~$B \cup \error_\calAR$.
A~state~$s$ is $B$-winning if the tester has a winning strategy from~$s$.
The set of winning states is exactly $\CPre_\calAR^*(B)$:

\begin{lemma}
  \label{lemma:mps}
  For all~$B\subseteq S^\inp$, there exists a strategy $\st$
  such that from all states~$s \in \CPre_\calAR^*(B)$, all $\Outcome(\st,s)$ reach
  $B \cup \error_\calAR$.
\end{lemma}

The strategy $\st$ provided by Lemma~\ref{lemma:mps} is a \emph{winning  strategy}
in the game $(\calAR, B \cup \error_{\calAR})$.

One could look for a winning strategy in $(\calAR, O \cup \error_{\calAR})$ in order to use it as a test algorithm.
This approach was considered in some works (\eg, in~\cite{DLLN08}).
However, in most cases, $\CPre_\calAR^*(O)$ does not contain the initial state; so such a strategy cannot be applied from the beginning.
In terms of game theory, this means that the tester does not have a winning strategy from the initial state.
We now discuss what to do if the initial state is outside $\CPre_\calAR^*(O)$.

\paragraph{Cooperative Steps}
Given $B\subseteq S^\inp$, consider $\Pre_\calAR(B)$, the set of immediate predecessors of~$B$:
by definition, from these states, there exists a pair of valuations $(\nu^\inp, \nu^\out)$ for which the successor state is in $B$.
From these states the tester cannot always guarantee reaching~$B$ in one step; however, it can choose an input valuation
$\nu^\inp$ for which \emph{there exists} $\nu^\out$ which moves the system into~$B$. In fact, when attempting to reach~$B$,
if the current state is not in $\CPre_{\calAR}^*(B)$, then choosing such a $\nu^\inp$ might be a good choice.

A strategy $\st$ is \emph{$B$-cooperative} if for all $s\in \Pre_{\calAR}(B)$, and all $\nu^\inp \in \st(s)$,
there exists $\nu^\out$ such that  $\Post_{\calAR}(s,{\vAP}^\inp\cup {\vAP}^\out) \in B$.
%In other words, by choosing $\vAP^\inp$ in $\st(s)$,
%even though there is no guarantee to reach $B$ in one step,
%there is a chance to do so if the system cooperates,
%while taking $\vAP^\inp$ outside $\st(s)$ ensures to fall outside $B$ in one step.
% \begin{example}
As an example, in Fig.~\ref{figW_i}, in $s^c_0$, the $O$-cooperative strategy $\st$ is represented by a bold arrow. Choosing this input valuation \emph{might} lead to $O$ at the next step, while choosing the input valuation outside $\st(s)$ surely leads outside $O$ in the next step, for any output valuation.
% \end{example}

In the rest of this section, we will combine winning and cooperative strategies to 
define \emph{greedy} strategies, which can be seen as best-effort strategies
that can be employed when there are no winning strategies from the initial state.
These guarantee progress towards~$O$ in~$\calAR$ against \emph{some} system strategies:
because the implementation is arbitrary, the system strategy may not be ``cooperative'',
that is, 
at state $s\in \Pre_{\calAR}(B)$, given input $\nu^\inp \in \st(s)$,
it may not output $\nu^\out$ ensuring $\Post_{\calAR}(s,{\vAP}^\inp\cup {\vAP}^\out) \in B$.
Nevertheless, we will show how to benefit from these strategies using randomization,
using them as a heuristic to improve over the basic MCTS test algorithm.

%Consider the set of requirements $\calR=\{R_i\}_{i \in I}$.
%Each requirement $R_i$ in $\calR$
%is specified by some automaton $R_i=\tuple{S_i, {s_0}_i, AP, T_i, F_i}$
%with error state $\error_i$ and accepting states $F_i=S_i\setminus \error_i$.

%% We~first build the product automaton $\calR=\tuple{S, {s_\init}, AP, T, F}$
%% of the automata~$R_i$ corresponding to the set~$\calR$ of requirements,
%% and consider a non-empty subset~$B$ of $S$
%% as the test objective.
%\todo{TJ: replaced $F=\times_{i\in I} F_i$ to include $Error$ states}

\paragraph{Greedy Strategy} Consider~${W_0= \CPre_\calAR^*(O)}$, which is non-empty 
since it contains~$O$. 
% \todo{TJ: Notice that if $O$ is a set of error states, $W_0$ is exactly $O$ since the tester cannot control errors}
% OS: I believe W_0 can still be different than O when err is in O.
Let~$\st_0$ be a $O$-winning strategy given by Lemma~\ref{lemma:mps}.
If $s_\init$ belongs to~$W_0$, then 
$\st_0$ ensures reaching~$O \cup \error_{\calAR}$ from~$s_\init$ against all system strategies.
In this case, we stop and return $\st_0$ as the greedy strategy.

The interesting and more frequent case is when $s_\init \not \in W_0$.
In this case, we~inductively define an increasing sequence of sets of states 
$W_0,W_1,\ldots,W_n$
and corresponding strategies such that $W_n$ is the set of all coreachable states.

Consider~$i\geq 0$, and assume that the sequence has been defined until index $i$.
%Let $S_{i} = \Coop_0 \cup W_0 \cup \Coop_1\cup W_1 \cup \ldots \cup\Coop_i \cup W_i$.

Let~$\Coop_{i+1}= \Pre_\calAR(W_i) \setminus W_i$ 
be the set of immediate predecessors of $W_i$ deprived of states that have already been seen.
%If $\Coop_{i+1}$ is empty, we stop (all coreachable states have been covered). %and no strategy is generated since $O$ cannot be reached from $s_\init$. 
Define $\st^c_{i+1}$ as a $W_i$-cooperative strategy.
Then, let $W_{i+1} = \CPre_\calAR^*(\Coop_{i+1} \cup W_i)$, and consider
a corresponding winning strategy $\st^w_{i+1}$.
Let us denote by $\st_{i+1}$ as follows:
for all~$s\in S^\inp$, if $s\in \Coop_{i+1}$, then $\st_{i+1}(s) = \st^c_{i+1}(s)$;
if $s \in W_i$, then $\st_{i+1}(s) = \st_i(s)$;
and if $s \in W_{i+1}\setminus (\Coop_{i+1}\cup W_i)$, then $\st_{i+1}(s) = \st^w_{i+1}(s)$ ($\st_{i+1}$ is defined arbitrarily elsewhere).

We stop this sequence whenever $\Coop_i$ becomes empty which happens when 
\(W_i = \coreach(\calAR,O)=\Pre_\calAR^*(O)\).

\begin{figure*}[htb]
  \centering
  \includegraphics[width=12.5cm]{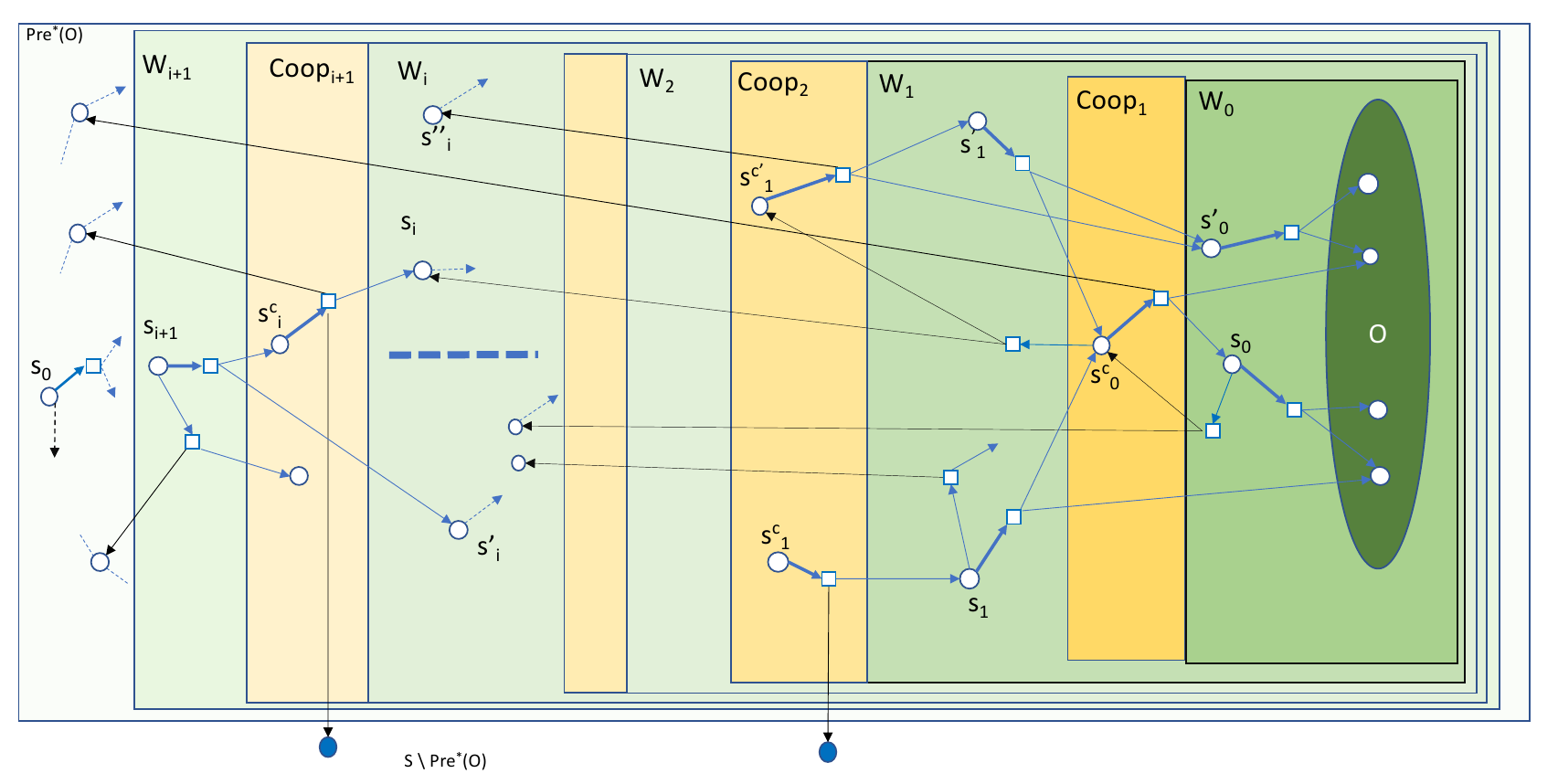}
  \caption{Construction of the $(W_i)_i$ hierarchy and greedy strategies.
  States in $S^\inp$ where the tester plays are represented by circles,
  while transient states in $S^\out$ where the system plays are represented by squares. 
  In a state $s_i$ in $W_i$, the strategy $\st_i$ is illustrated by bold blue arrows. 
  Thin black arrows represent those outputs that either reach a state in a set $W_k$ with higher rank
  or reach $S \setminus \coreach(\calAR,O)$.
  For~example, $s_1  \in W_1$ since the bold arrow ensures progressing to $\Coop_1\cup W_0$.
  However, $s_2 \in \Coop_2$ since the bold arrow \emph{can} make the game progress to $W_1$ but it can also 
  go outside $\coreach(\calAR,O)$ depending on the system strategy.
  }
  \label{figW_i}
\end{figure*}

The construction thus builds an increasing hierarchy $(W_i)_i$ of larger and larger sets, where at each level $i$,
cooperation is needed in states of $\Coop_i$ to get closer to~$O$ in this hierarchy.
% All states of $\coreach(\calAR,O)$ belong to some $W_i$.
%in~particular, if $O$ is reachable from~$s_\init$ (\ie, $s_\init \in\coreach(\calAR,O)$),
%there exists some index~$k$ such that $s_\init$ in $W_k$.
For a state $s \in \coreach(\calAR,O)$, we call the \emph{rank} of $s$ and denote $\rank(s)$ the smallest index $i$ such that $s \in W_i$.
We denote by $\stgreedy$ the strategy defined by 
$\stgreedy(s) = \st_{\rank(s)}(s)$ for all $s \in \coreach(\calAR,O)$;
and arbitrarily for other states.
We call this the \emph{greedy strategy}.

Notice that none of the winning strategies ($\st^w_i$ or $\st_0$) composing $\st$ induces loops.
However, since $\st$ is also composed of cooperative strategies, outcomes may contain loops due to outputs reaching states with higher ranks in the $(W_i)_i$ hierarchy. 
In~other terms, $\stgreedy$ guarantees reaching~$O\cup\error_{\calAR}$ for some system strategies but not all; so it may induce infinite loops against some system strategies.
Intuitively, the strategy $\stgreedy$ requires minimal cooperation from the system but it is also very optimistic.

The following example shows that the greedy strategy does not always guarantee progress towards the objective at each
step. We will still use the greedy strategy in different ways to improve online testing and evaluate its effectiveness empirically.

\begin{example}
  \label{example:bad-example-for-greedy}
  % The greedy strategy does not always guarantee
  % progress towards the reachability objective.
  In Fig.~\ref{fig-exaut}, 
  we have $W_0 = \{o\}$, $\Coop_0=\{s_1\}$, $\Coop_1=\{s_0\}$
  while $\stgreedy(s_0) =\{a\}$ and $\stgreedy(s_1) =\{b,c\}$.
  The implementation given in Fig.~\ref{fig:impl-bad-for-greedy} admits
  a covering trace which starts with $(b,0)$. However, because 
  $b \not \in \stgreedy(s_0)$, this cannot be found by the greedy strategy.
\end{example}
\begin{figure}
  \centering
  \begin{tikzpicture}[xscale=2.3, yscale=1.3,initial text={},xscale=0.45]
    \draw (-0.4,0) node (ini) {};
    \draw (0,0) node[rond3] (s0) {} node {}; 
    \draw (0.5,0)  node[carre3] (s0out) {} node {}; 
    \draw (1,0) node[rond3] (s1) {} node {}; 
    \draw (1.5,0)  node[carre3] (s1out) {} node {}; 
    \draw (2,0) node[rond3] (s2) {} node {}; 
    \draw (2.5,0)  node[carre3] (s2out) {} node {}; 
    \draw (3,0) node[rond3] (s3) {} node {}; 
    \draw (1.5, 0.7) node[carre3] (todummy) {};
    \draw (1.5, 1.2) node[rond2] (dummy) {};

    \begin{scope}[-latex']
      \draw (ini) edge (s0);
      \draw 
        (s0) edge node[below]{b} (s0out)
        (s0out) edge node[below]{0} (s1)
        (s1) edge node[below]{a} (s1out)
        (s1out) edge node[below]{0} (s2)
        (s2) edge node[below]{b} (s2out)
        (s2out) edge node[below]{1} (s3)
        (s0) edge[bend left] node[left] {a,c~} (todummy);
      \draw
        (s3) edge[bend right] node[right] {~~a,b,c} (todummy)
        (todummy) edge[bend left] node[left] {1} (dummy)
        (dummy) edge[bend left] node[right] {a,b,c} (todummy)
        (s1) edge node[left]{b,c} (todummy)
        (s2) edge node[right]{a,c} (todummy)
        ;
    \end{scope}
  \end{tikzpicture}
  \caption{An implementation 
  %on input alphabet $\{a,b,c\}$, and output alphabet $\{0,1\}$, 
  with the only covering trace $(b,0), (a,0), (b,1)$ for objective $\{o\}$
  of the automaton of Fig.~\ref{fig-exaut}. 
  }
  \label{fig:impl-bad-for-greedy}
\end{figure}

\paragraph{Purely greedy and $\epsilon$-greedy test algorithms}
The greedy strategy can be used to define a randomized test algorithm, called the \emph{pure greedy} algorithm, and denoted $\greedytc$, as follows.
Recall that after an active trace $\sigma$, the automaton $\calAR$ is in a state $s_\sigma \in \coreach(\calAR,O)$.
% In the uniform strategy an input $\vAP^\inp$ is chosen uniformly in  $\{\vAP^\inp \mid (s_\sigma,\vAP^\inp) \in \coreachinp(\calAR,O)\}$.
% A simple modification consists in replacing this choice with
We let $\greedytc$ pick an input from $\st(s_\sigma)$ uniformly at random. 
% If $s_\sigma$ is in $W_i$ for some $i$, following this strategy will inevitably lead either to $\Coop_i$ if $i>0$, and to $O$ if $i=0$, or possibly to $\error$ if $I$ is non-conformant.
% Once in $\Coop_i$, the tester uses $\stgreedy$ because there is a possibility of entering $W_{i-1}$;
% but the implementation, even if conformant, is not forced to be cooperative, and may go back to some $W_j$ with $j\geq i$.
For $K>0$, let $\greedytc_K$ denote the test algorithm obtained from $\greedytc$ by stopping each run after~$K$ steps, and restarting another run.

The algorithm $\greedytc_K$ may not be almost-surely complete for any~$K>0$ (Example~\ref{example:bad-example-for-greedy}).
However, we~can obtain a almost-surely complete algorithm 
$\epsgreedytc$ simply by randomizing between $\uniformtc$ and $\greedytc$:
given $0<\epsilon<1$, this algorithm 
uses, at each state $s_\sigma$, 
$\greedytc$ with probability $1-\epsilon$; and $\uniformtc$ with probability $\epsilon$.
This ensures theoretical almost-sure completeness since 
if $O$ is reachable in $\calAR \otimes I$, 
then, at each run, there is a positive probability that the $\epsgreedytc_K$
executes $\uniformtc$ for $K$ steps, while picking the right inputs, and this will occur with probability 1
if these runs are repeated.

\paragraph{Greedy-MCTS Test Algorithms}
Here, we explain our main contribution which is an improvement
of the Basic MCTS test algorithm of
Section~\ref{sec-mcts} using the greedy strategy as heuristics both in the roll-out- and tree policies.
Our heuristic accelerates convergence: it favors
inputs that tend to get closer to the objective over those that do not,
and if possible, more rarely use inputs that may lead to inconclusive
states.

The first modification we make is using $\epsgreedytc_K$ as the roll-out policy
instead of $\uniformtc_K$, for some given~$K$.
We call this \emph{MCTS+GRO} (for Greedy ROllout).
The second modification consists in using the greedy policy within the tree policy.
Fix a bound~$M>0$.
During the expansion phase, in a given node whose projection in $\calAR$ is $s$,
we restrict the UCT policy to the inputs allowed by the greedy strategy $\stgreedy(s)$
at the first $M$ visits to that tree node, and then fallback to the regular UCT policy which covers
the whole set of input valuations. 
We call this \emph{MCTS+GTRO} (for Greedy Tree and ROllout).
The bound $M$ applies separately to each node, so at any moment, there are always nodes that have been explored less than $M$ times at which the tree policy is restricted to the inputs of the greedy strategy.

Because the restriction to the greedy inputs only holds for a finite number of visits, this does not affect the convergence guarantees of the MCTS algorithm.
Our heuristic is intended to improve convergence by introducing a bias.
Similar biased UCT scores have been used \textit{e.g.} to solve the game Go~\cite{gelly2011monte}.

\begin{table*}[t]
  \centering
  \caption{Average execution times in seconds and success rates (over 100 experiments). 
    Percentages show the number of successful trials over 100 experiments within the time limit (others stopped without a result after the time out).
    An absence of a percentage in a cell means 100\% success rate.
    Given data are for $\epsilon=1/8$ (the left of the double column), and $\epsilon=1/4$ (the right of the double column).
    Note that $\epsilon$ is not relevant for $\uniformtc_K$ and naive MCTS so these are not reproduced on the right.
    A dash means that no trace was found within the time out. GRO stands for greedy roll-out, and GTRO for greedy tree policy and roll out.
    The bold entries show the best performing algorithm(s) on each row among algorithms with 100\% success rate.
    \vspace{0.2cm}
  }
  \scalebox{0.68}{
  \begin{tabular}{|l|c|c|c|c|c|||c|c|c|}
    \hline
    \bfseries       &  $\uniformtc_K$           & $\epsgreedytc_K$         & MCTS                     & MCTS + GRO             & MCTS + GTRO             & $\epsgreedytc_K$\newline     & MCTS + GRO        & MCTS + GTRO        \\\hline
    FTSP 3          & -                         & 226.9 \scriptsize(67\%)  & 275.6 \scriptsize(88\%)  & 138.7                  & \bf 77.6                & 210.6                        & 178.3                   & 131.8     \\\hline 
    FTSP 4          & -                         & -                        & 407.0 \scriptsize(23\%)  & 456.0                  &  \bf 290.7              &  -                           & 438.9                   & 341.9    \\\hline 
    Ring 16         & -                         & 0.6                      & 96.3                     & 0.7                    & 0.7                     & 0.9                          & \bf 0.4                 & 0.9   \\\hline
    Ring 32         & -                         & \bf 1.1                  & 216.3 \scriptsize(94\%)  & 1.1                    & 1.2                     & 56.4                         & 59.4                    & 16.3    \\\hline 
    Passage 3       & 113.2 \scriptsize(98\%)   & 6.0                      & 7.3                      & 1.8                    & \bf 0.5                 & 6.7                          & 2.0                     & 1.2     \\\hline
    Passage 5       &  -                        & 169.7 \scriptsize(66\%)  & 47.1                     & 17.8                   & \bf 2.8                 & 218.2 \scriptsize(41\%)      & 20.7                    & \bf 2.8     \\\hline
    Passage 10      &  -                        & -                        & 429 \scriptsize(3\%)     & 300 \scriptsize(9\%)   & \bf 15.0                & -                            & 327 \scriptsize(14\%)   &  15.2      \\\hline
    FTSP' 3         & 72.3                      & 51.8                     & 64.2                     & 44.2                   & 41.0                    & 47.9                         & \bf 36.9                & 47.4       \\\hline 
    FTSP' 4         & 126.7 \scriptsize(3\%)    & 105.5 \scriptsize(3\%)   & 468.5 \scriptsize(34\%)  & 419.6 \scriptsize(34\%)& 438.6 \scriptsize(25\%) & 85 \scriptsize(2\%)          & 434.9 \scriptsize(36\%) & 459.3 \scriptsize(30\%)         \\\hline 
    FTSP'' 3        & -                         & 217.0 \scriptsize(68\%)  & 303.2 \scriptsize(69\%)  & 179.7 \scriptsize(98\%)& \bf 110.6               & 222.6 \scriptsize(11\%)      & 220.4 \scriptsize(95\%) & 154.5 \scriptsize(98\%)         \\\hline 
    FTSP'' 4        & -                         & -                        & 416.6 \scriptsize(8\%)   & 498.8 \scriptsize(2\%) & 364.5 \scriptsize(77\%) & -                            & 427.7 \scriptsize(15\%) & 356.6 \scriptsize(73\%)         \\\hline 
  \end{tabular}
  }
  \label{table:results}
\end{table*}
\section{Experiments}
We implemented the test generation algorithm in Python where requirements are specified as deterministic automata
in Verilog. These requirements are automatically translated to the AIGER format
using \href{https://github.com/YosysHQ/yosys}{Yosys} and \href{https://github.com/berkeley-abc/abc}{ABC},
which is then read by
the \href{https://github.com/gaperez64/AbsSynthe}{Abssynthe} game solver~\cite{BrenguierPRS14}.
We implemented the computation of the sets $(W_i)_i$ and the greedy strategy of Section~\ref{section:greedy} within Abssynthe.
The testing tool  uses the \href{http://vlsi.colorado.edu/~fabio/CUDD/}{CUDD} BDD library, 
and \href{https://github.com/mvcisback/py-aiger}{pyAIGER}
to read and execute the greedy strategy and the sets $W_i$ computed by Abssynthe, and to compute rewards.
The tool and experiments are open-source and can be found at \url{https://github.com/osankur/rl-tester}\footnote{The artifact for this paper can be downloaded from \url{https://doi.org/10.5281/zenodo.21507486}}.

We apply the presented testing algorithms on several instances of three systems:
the flooding-time synchronization protocol (\textit{FTSP}) with 3 and 4 processes~\cite{maroti2004flooding,sankur2017abstraction}
where the test objective is to find an execution where all processes correctly elect a leader,
a ring leader election protocol (\textit{Ring}) with 16 and 32 processes \cite{chang1979improved}, 
and a control program solving a planning problem (\textit{Passageway}).
We wrote our own implementations and requirement automata for these systems.
% Appendix~\ref{section:benchmarks} contains detailed descriptions of these.

\paragraph{Algorithms}
We compared all presented algorithms.
These include the baseline algorithms $\uniformtc_K$ and the basic~MCTS,
along with our algorithms $\epsgreedytc_K$ and MCTS with greedy roll-out (GRO) and MCTS with greedy tree policy and greedy roll out (GTRO).

We do not compare with \cite{BEMN-oopsla24} since the approach is not automatized and cannot consider automata-based requirements (but only user-provided waypoints). 
Further, note that in basic MCTS, we use rewards that are based on remaining distances in our hierarchy. This can be seen as a direct application of \cite{camacho2019ltl} to model-based testing, where a reinforcement learning was applied with a reward function based on the distance to the goal. In fact, we see the basic MCTS as a representative of the approach of using an RL-based algorithm with a clever choice of rewards. Last, we also do not compare with \cite{DBLP:journals/stvr/KorogluS21} whose reward functions are highly dependent on LTL formulas; but their RL-based approach can be seen as being represented by our basic MCTS.

We initially used Torxakis as a baseline but do not present these results here. The reason is that it failed to reach the test objective
even for small instances so the performance was systematically below our implementation of $\uniformtc_K$.
We believe that this might be because Torxakis does not analyze the state space of the requirements automaton so it does not eliminate inputs that lead to inconclusive states.
Furthermore, the input language is very complex and it was difficult to interpret the error messages of the compiler, so we only made the modeling effort of Passageway.

%For~uniform and greedy algorithms, this meant that we made a total number of $100\times 50,000 = 5,000,000$ runs, each making 1000~steps.

\paragraph{Parameters}
Each algorithm made up to 50,000 runs where each run generates a trace of length $K=1000$, with an overall time-out of 600s.
We repeated each experiment 100 times in order to compute average running times.
For example, for~the basic~MCTS, we started from scratch 100 times, and built the tree up to 50,000 nodes, making a roll-out of length~1000 for each added node.

In these experiments, we set $\epsilon=1/8$ and $M=1000$ which were empirically determined to give best results.
In fact, we also experimented with lower values of $M$ (from $0$ to $1000$) and the highest value was best in our cases,
and it was not useful to increase $M$ any further. In our experiments, low values of $\epsilon$ performed better, while decreasing $\epsilon$ 
any further did not considerably change the outcome.
The choice of $K=1000$ is somewhat arbitrary. The results were not impacted by choosing larger values (while choosing significantly smaller values meant that objectives could never be reached).

We report here the running times and success rates once these parameters were fixed.
As in many applications of reinforcement learning, there is no obvious \emph{a priori} choice of the parameters,
and finding good parameters is often part of the process.
In practice, $K$ can be chosen if one has a good guess for an upper bound on the length of executions.
As for the choice of $\epsilon$ and $M$, one could start with small values for $\epsilon$ and large values of $M$ which will strongly favor
greedy strategies, and decrease gradually in case this does not work.
In our results, we show execution times for $\epsilon=\frac{1}{4}$ as well to assess the sensitivity of the performance on the choice of $\epsilon$.

\paragraph{Results}
The results of various algorithms on the described systems are given in Table~\ref{table:results}.
We comment here on the FTSP benchmarks. %, and leave the discussion on other benchmarks in the appendix.
FTSP is a leader election algorithm where processes run in quasi-synchrony (they are asynchronous but in a bounded way - see \cite{sankur2017abstraction}),
can crash, and recover. The tester selects at each step the process to execute, and whether the selected process must crash.
The requirements for FTSP ensures this property of the scheduling except that a process recovering from a fault can execute more often. Moreover, the automaton tracks which process has correctly elected a leader, and which one is ready to broadcast the leader information (which will in turn ensure other processes elect the leader).
In FTSP, along a normal execution each process can be in three phases: A) initially the process has no leader (or has the wrong leader),
B) eventually receives a message about the correct leader, and C) eventually gathers enough messages from the leader so that it starts broadcasting itself.
Modelling these three phases for each process in the requirements helped the test algorithm guide the executions towards the objective.

As most benchmarks, FTSP was too hard for $\uniformtc_K$, but introducing learning with MCTS or using greedy strategies were successful:
more precisely, on FTSP 3, basic MCTS and $\epsgreedytc_K$ both succeeded but with less than 100\% rate.
A success rate of 100\% is achieved by combining both
while MCTS + GTRO was the fastest. The same success rate but higher execution times were achieved with $\epsilon=1/4$.
The model FTSP 4 is larger which can be seen on the lower success rates and higher execution times.

We also experimented with two variants of the FTSP benchmark (FTSP' and FTSP'') where the implementation was identical but the requirements automata were different.
In FTSP', the requirements strictly ensure the quasi-synchrony condition regardless of the recovery phase. 
This meant that inputs violating this condition now necessarily lead outside of $\coreach(\calAR,O)$, and are thus removed from the 
set of possible inputs for all algorithms. Inconclusive states were therefore no longer reachable. 
This is why $\uniformtc_K$ performs relatively well in these cases.

FTSP'' is an example showing the limits of the approach and the importance of selecting a detailed model. 
Here, the implementation is the same as FTSP, but the requirements only track
the scheduling constraint and whether each process has elected a leader, but not whether a process is ready to broadcast or not (phase C above).
Thus rewards computed from the requirements carry less information, which explains lower success rates.
In particular basic MCTS and MCTS+GRO rarely find solutions for FTSP'' 4, while MCTS+GTRO achieves success rate of 77\%: compare this with 100\%
success rate for FTSP 4. %where we solve the same problem but with a slightly more detailed requirement automaton.
This is consistent with \cite{BEMN-oopsla24} where handcrafted waypoints (intermediate objectives) were necessary to achieve coverage and good convergence time.
In our cases, these ``waypoints'' correspond to the states of the requirement automaton.

The advantage of the greedy strategies can be seen clearly in the ring protocol. Here the tester selects at each step the type of crash that should occur
in the system, so the greedy strategies are useful to avoid selecting crashes that delay reaching leader election. In this case, learning had some advantage,
especially for $\epsilon=1/4$, but we believe that the success is mainly due to the greedy strategy which focuses on useful inputs.
%Appendix~\ref{section:benchmarks} contains more details about the benchmarks.

The preprocessing times (time required to solve the game, and build the hierarchy) were below 1s in each case with the following exceptions: 2.5s for FTSP4$'$ and 31s for FTSP4 (with some variability due to randomization in the BDD package).
The largest BDD sizes observed were for the FTSP4 benchmarks:
for FTSP4, FTSP4$'$, FTSP4$''$, we measured, respectively, 74927, 9231, 468 nodes.

All requirement automata are encoded as Verilog modules so we report below the number of registers (bits) used in these modules. 
This gives an upper bound on the size of the automata (for $n$ bits, there can be up to $2^n$ states) although the reachable part of the state space might be smaller.

\begin{center}
\begin{tabular}{lc}
  FTSP3 and FTSP3' & 38 bits\\
  FTSP3'' & 28 bits\\
  FTSP4 and FTSP' & 46 bits\\
  FTSP4'' & 36 bits\\
  Ring & 14 bits\\
  Passageway & 10 bits
\end{tabular}
\end{center}

Peak memory consumption (RSS) measured with gnu time was 57MB for Ring32, 750MB for FTSP4, and 57MB for Passageway10.

\section{More on Related Works}
The notion of cooperative states have been used in~\cite{HJM18} in the context of offline test synthesis from timed automata, inspired by the previous approach of \cite{Ramangalahy98}.
In~\cite{DLLN-MBT08}, in the context of timed systems, the~authors rely on cooperative strategies to synthesize test cases when the cooperation of the system is required for winning. 
However, these yield incomplete testing methods ,
or~completeness is obtained by making strong assumptions on the implementations; the~novelty of our approach is to obtain a complete method by using these notions
in reinforcement learning.
A~discussion on model-based testing techniques can be found in~\cite{Veanes2008}.

Several test generation tools based on the {\bfseries ioco} conformance relation for input-output labelled transition systems~(IOLTS) have been developed. Roughly, an~implementation ioco-conforms to its specification if after any of its observed behaviour that is also a specification behaviour,
the implementation only produces outputs or quiescences that are also possible in the specification.
The~tools TGV~\cite{jard2005tgv} and TESTOR~\cite{MMS-tacas18} generate off-line test cases from formal specifications in various languages with IOLTS semantics, driven by test purposes described by automata. 
The tools JTorx~\cite{belinfante2010jtorx} and TorXakis~\cite{torxakis} are 
improvement of TorX~\cite{tretmans2003torx} and allow to generate and execute online test cases derived
from various specification languages.
In~the context of timed models, 
Uppaal-TRON~\cite{hessel2008testing} is an online test generation tool for timed automata based on the real-time extension {\bfseries rtioco} of the {\bfseries ioco} conformance relation. 

In \cite{BEMN-oopsla24}, the authors develop a testing framework based on Q-learning with dynamic rewards that encourage a good coverage of a system,
and apply their methods to test distributed algorithms. In particular, the Cube World example they develop is similar in spirit to our passageway system.
Moreover, they guide the reinforcement learning agent using \emph{waypoints} which are important steps that mark progress towards the objective in the analyzed system.
Our automata-based framework can be seen as a generalization of waypoints since requirements can express sequences of particular events in the system,
but also more complex behaviors.

\cite{MPRS-icse11} uses Q-learning to produce tests for GUI applications but without a model for the specification:
the objective is to reach a large number of visually different states.
RL-based testing for GUI has attracted significant attention. \cite{adamo2018reinforcement}~uses Q-learning with the aim of covering as many states as possible; see~also~\cite{LPSLY-ase22}.
In~\cite{reddy2020quickly}, reinforcement learning is used to compute \emph{valid} inputs for testing programs:
these consist in producing inputs that satisfy the precondition of a program to be tested so that assertions can be checked.
\cite{THMT-ase2021} uses RL to learn short synchronizing sequences, where rewards correspond to the size of the
powerset of states.
In \cite{PZAdSL-date2020}, reinforcement learning was used to test shared memory programs.
% MCTS has been used for testing in various settings. 
% In \cite{ariyurek2020enhancing},
% MCTS is used for testing video games using rewards to cover different areas in the game,
% but without automata specifications.
Deep reinforcement learning has also been used for Android testing; see e.g.~\cite{RMCT-acm2021}.
\cite{FCP-mcs23} combines blackbox testing and model learning in order to improve coverage.

Reward shaping for automata-based specifications has been considered for Monte Carlo Tree Search.
In \cite{camacho2019ltl}, the approach is based on the distance to accepting states of B\"uchi automata;
and in~\cite{velasquez2021dynamic}, the authors collect statistics on the success for all transitions on the specification automaton.
The latter approach is not adapted to our case, where the goal is to find a single successful execution, and not to actually learn the
optimal values at all states. 

\section{Conclusion}
While the game point of view has been explored before in requirement-based testing, we used it to improve the reinforcement learning approach.
Our preliminary experiments show that such heuristics can improve the testing time by guiding the search quickly to
relevant states, especially when test objectives require long sequences that have low probability to be guessed randomly.

Our method applies to systems under test interacting with an environment by alternating inputs and outputs, so the system must be instrumented
to interact with the test driver. The method moreover needs requirement automata to be provided, which are not always obvious to build. 
While we wrote requirements of our benchmarks directly as automata,
one might also obtain automata from requirements written in structured natural language~\cite{giannakopoulou2020formal} or dedicated languages based on 
temporal logic~\cite{BeckerDHHLPW21,SNJM-reqkit25}.
Moreover, there is a large body of work on the TLA+ formalization of distributed algorithms (\emph{e.g.} \cite{BGRSS-nfm25}),
and many works report the benefits of testing w.r.t. such formal requirements \cite{DHS-vldb20}.
We will investigate the extension of our approach to such specifications in future work.
The definition of games on such formalisms and extending game solvers to such languages are difficulties that need to be addressed.

Our approach consists in testing behaviors against temporal, i.e. multi-steps, specifications, which means functional testing of complex behaviors of a system under test (SUT). Our approach is insensitive to choices in the implementation, so it is suitable when SUT is changed, and can efficiently find relevant test cases (i.e. each test case being a list of successive input values) without user intervention.

We considered here system-level testing where requirements refer to the behaviors of the whole system.
One might also consider this approach for unit-testing components. In this case, the requirement automaton would refer to the behaviors
of the component which is isolated possibly using stubs following standard unit-testing practices.

\appendix
\section{Appendix}
\subsection{Omitted Definitions}
\label{app:definitions}
\begin{definition}
  Given automata
  $\calA_1=\tuple{S_1^\inp\uplus S_1^\out, s_{\init,1}, AP_1, T_1, F_1}$ and
  $\calA_2=\tuple{S_2^\inp\uplus S_2^\out, s_{\init,2}, AP_2, T_2, F_2}$,
  their \emph{product}
  $\calA_1 \otimes \calA_2$ is an automaton
  $\calA=\tuple{S, s_\init, AP, T, F}$ where
  $S=(S_1^\inp \times S_2^\inp) \uplus (S_1^\out\times S_2^\out)$,
  $s_\init= (s_{\init,1},s_{\init,2})$,
  $AP=AP_1 \cup AP_2$,
  $F= F_1 \times F_2$
  and the set of transitions is defined as follows:
  there is a transition $((s_1,s_2), g, (s'_1,s'_2))$ in~$T$ if there
  are transitions $(s_1, g_1, s'_1)$ in $T_1$ and $(s_2, g_2, s'_2)$
  in $T_2$ with $g=g_1 \et g_2$.
\end{definition}

% We now consider an example of an automaton that will be used to illustrate other notions we define later in the paper.
% \begin{example}

Notice that this definition indeed yields an automaton in the sense of
Def.~\ref{def-automata} and preserves completeness and
determinism.
Moreover, the product of two safety automata is a safety automaton: the set of accepting states is $F= F_1\times F_2$, so~the~set~$\error$ in the
product automaton is $(\error_1 \times S_2^\inp) \uplus
(S_1^\inp \times \error_2)$, and thus inherits absorbance.
The~product of automata is commutative and associative, and can thus
be generalized to an arbitrary number of automata.

\subsection{Additional Examples}
\label{app:examples}

The next example shows how an automaton can be obtained from requirements for a simple factory automation system.
\begin{example}[Factory Carriage Example]
  \label{example:carriage}
  We consider a controller program in a factory automation system depicted in Fig.~\ref{fig:carriage}.
  In this system, when the carriage is on the right end (\texttt{bwdlimit}) and receives a \texttt{cargo} on top of~it,
  the controller program must move the carriage forward (\texttt{movefwd}) 
  until it reaches the forward limit (\texttt{fwdlimit}).
  % Then it pushes the arm for 3 seconds.
  %(\texttt{push\_arm}) 
  %until a timer expires (this models a pushing time of 3 seconds).
  The controller must then push the arm for 3 seconds, and it can only back the carriage up (\texttt{movebwd}) 
  after this duration.
  \begin{figure}[ht]
    \centering
    \includegraphics[scale=0.35]{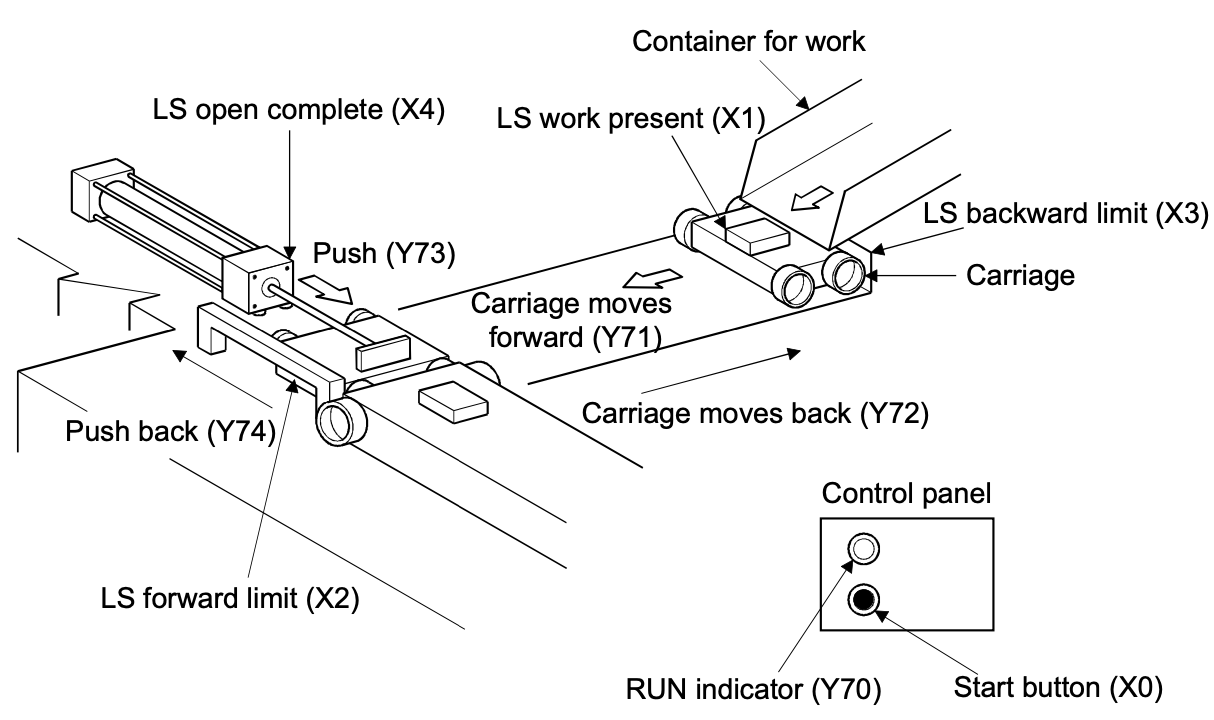}
    \caption{The carriage control system example from~\cite{PLC-mitsu}.}
    \label{fig:carriage}
  \end{figure}

  For the sake of this example, we only model the requirements concerning the first phase, that is, until the carriage reaches the forward limit.
  Three of these requirements on the controller program are given below.
  \begin{itemize}
    \item[R1] When the carriage is on its backward limit, and a cargo is present, then it immediately moves forwards until reaching the forward limit.
    \item[R2] If the carriage is not already moving forward, and if no cargo is  present or the carriage is not in the backward limit, then it is an error
    to move forward. The carriage must never be moved forward and backwards at the same time.
    \item[R3] When some cargo is present, and the carriage is at its forward limit, it~should stop moving forward immediately.
  \end{itemize}

  All these three requirements are modelled in the automaton of Fig.~\ref{fig:carriage-automaton}. 
  The initial state is~$s_0$, and we distinguish the state \textsf{err} which makes this a safety automaton.
  Intuitively, state $s_2$ is reached when the carriage receives a cargo and brings it successfully to the forward limit.

  \begin{figure*}[ht]
    \centering
    \begin{tikzpicture}[xscale=1.6, initial text={}]
      \draw (0,0) node[initial,rond6] (s0) {} node {$s_0$}; 
      \draw (-1.2,-1)  node[carre3] (s0out2) {} node {}; 
      \draw (1.2,-1)  node[carre3] (s0out3) {} node {}; 
      \draw (1.5,0)  node[carre3] (s0out1) {} node {}; 
      \draw (0,-2) node[rond6] (err) {} node {$\mathsf{err}$}; 

      \draw (2.5,0) node[rond6] (s1) {} node {$s_1$}; 
      \draw (4,0)  node[carre3] (s1out1) {} node {}; 
      \draw (4,-1) node[carre3] (s1out2) {} node {}; 
      \draw (2.5,-2) node[accepting,rond6] (s2) {} node {$s_2$};

      \begin{scope}[-latex']
        \draw (s0) edge node[above]{$\cargo \land \bwdlimit$} (s0out1);
        \draw (s0) edge node[left]{$\lnot \cargo \land \bwdlimit$} (s0out2);
        \draw (s0) edge node[right]{$\lnot \bwdlimit$} (s0out3);
        \draw (s0out1) edge node[above]{$\movefwd$} (s1);        
        \draw (s0out3) edge  node[above]{$\true$} (err);
        \draw (s0out2) edge node[below]{$\substack{\movefwd\\ \lor \movebwd}$~~} (err);
        \draw (s0out2) edge[bend right] node[right]{$\substack{\lnot \movefwd\\ \land \lnot \movebwd}$} (s0);
        \draw (s1) edge node[above]{$\lnot \fwdlimit$} (s1out1);
        \draw (s1out1) edge[bend left] node[below]{$\movefwd$} (s1);
        \draw (s1) edge[bend right] node[below]{$\fwdlimit$} (s1out2);
        \draw (s1out2) edge node[below]{$\substack{\lnot \movefwd\\ \land \lnot \movebwd}$} (s2);
      \end{scope}
      \draw (s0out1) -- (1.5,-1.5);
      \draw (1.5,-1.5) edge[left,-latex'] node[below]{$\lnot\movefwd$} (err);
    \end{tikzpicture}
    \caption{An automaton modeling requirements R1, R2, R3. For readability, we omitted some transitions in the figure:
    from all output states, there is an additional transition guarded by $\movebwd \land \movefwd$ to \textsf{err}.
    Moreover, from each of the two rightmost output states, the negation of the guard of the only leaving transition goes to \text{err} as well.
    For instance, from $s_1$, reading $\lnot \fwdlimit$ and then $\lnot \movefwd$, we end in \textsf{err}. Note that reading $\lnot \bwdlimit$ in $s_0$ is also an error since this is not supposed to happen in this system.}
    \label{fig:carriage-automaton}
  \end{figure*}

\end{example}

The following example illustrates the product of an implementation with its requirement automaton.
\begin{example}
  \label{example:different_I}
  We consider the requirement expressed by the automaton of
  Fig.~\ref{fig-exaut}, the~objective defined by the singleton $O=\{o\}$, and the implementation represented to the left
  of Fig.~\ref{fig-eximpl}. Their product is represented in the right of
  Fig.~\ref{fig-eximpl}.
  There is a covering trace in this case since the input sequence $ab$ generates 
  $(a0b1)$ in~$I_1$, reaching $o$ in~$\calA_\calR$.

  Assume now that $I_1$ is modified so that it outputs~$1$ on input~$a$
  from~$s_0^I$, then
  the product would go to a state of the form~$(t,s_0^I)$. If~$t$ is an
  \error\ state, then the implementation does not conform to the
  requirement; if~not, then the test is inconclusive since~$o$~is no longer reachable.

  On the other hand, consider an implementation outputs~$0$ on any input. Then any trace
  is an active trace although the implementation does not have a covering trace
  (in fact, the product cannot reach a state involving~$o$).
  %\todo[inline]{NM : a finir}

  \begin{figure*}[ht]
    \centering
    \begin{tikzpicture}[scale=1.5]
      \begin{scope}
        \path(-.5,1) node {$I_1$};
        \draw (0,0) node[rond6] (s0in) {} node {$s_0^I$}; 
        \draw[latex'-] (s0in.180) -- +(180:3mm);
        \draw (0,1) node[carre5] (s0a) {};
        \draw (1,0) node[carre5] (s0b) {};
        \draw (0,-1) node[carre5] (s0c) {};
        \begin{scope}[-latex']
          \draw
          (s0in) edge[bend left,-latex'] node[left] {$a$} (s0a)
          (s0in) edge[bend left,-latex'] node[above] {$b$} (s0b)
          (s0in) edge[bend left,-latex'] node[right] {$c$} (s0c)
          (s0a) edge[bend left,-latex'] node[right] {$0$} (s0in)
          (s0b) edge[bend left,-latex'] node[above] {$1$} (s0in)
          (s0c) edge[bend left,-latex'] node[left] {$0$} (s0in)
          ;
        \end{scope}
      \end{scope}
      %
      %% \begin{scope}[xshift=4cm]
      %%   \path(-.5,1) node {$I_2$};
      %%   \draw (0,0) node[rond6] (s0in) {} node {$s^I_0$}; 
      %%   \draw[latex'-] (s0in.180) -- +(180:3mm);
      %%   \draw (0,1) node[carre5] (s0a) {};
      %%   \draw (1,0) node[carre5] (s0b) {};
      %%   \draw (0,-1) node[carre5] (s0c) {};
      %%   \begin{scope}[-latex']
      %%     \draw
      %%     (s0in) edge[bend left,-latex'] node[left] {$a$} (s0a)
      %%     (s0in) edge[bend left,-latex'] node[above] {$b$} (s0b)
      %%     (s0in) edge[bend left,-latex'] node[right] {$c$} (s0c)
      %%     (s0a) edge[bend left,-latex'] node[right] {$0$} (s0in)
      %%     (s0b) edge[bend left,-latex'] node[above] {$0$} (s0in)
      %%     (s0c) edge[bend left,-latex'] node[left] {$0$} (s0in)
      %%     ;
      %%   \end{scope}
      %% \end{scope}
      \begin{scope}[xshift=3cm]
        \draw (0,0) node[oblong6] (s0in) {} node {$s_0,s^I_0$};
        \draw[latex'-] (s0in.135) -- +(135:3mm);
        \draw (2,0) node[oblong6] (s1in) {} node {$s_1, s_0^I$}; 
        \draw (4,0) node[oblong6,double] (s2in) {} node {$o, s_0^I$}; 
  %      \draw (2,2) node[rond6] (sink) {} node {$t$};
        \draw (1,0) node[carre5] (s0a) {};
        \draw (3,1) node[carre5] (s1b) {};
        \draw (3,-1) node[carre5] (s1c) {};
        \draw (1.7,-1.2) node[carre5] (s1a) {};
        \draw (-0.8,-1.2) node[carre5] (s0b) {};
        \draw (-1.2,-0.3) node[carre5] (s0c) {};
      \begin{scope}[-latex']
        \draw
        (s0in) edge[-latex'] node[above] {$a$} (s0a)
        (s0in) edge[-latex',bend left=10] node[below] {$b$} (s0b)
        (s0in) edge[-latex',bend left=10] node[below] {$c$} (s0c)
        (s1in) edge[bend left,-latex'] node[below right] {$a$} (s1a)
        (s1in) edge[-latex'] node[above] {$b$} (s1b)
        (s1in) edge[-latex'] node[above] {$c$} (s1c)
        (s0a) edge[-latex'] node[above] {$0$} (s1in)
        (s1a) edge[-latex',bend left] node[left] {$0$} (s1in)
        (s0b) edge[-latex',bend left=10] node[left] {$1$} (s0in)
        (s0c) edge[-latex',bend left=10] node[above left] {$0$} (s0in)
  %      (s1b) edge[-latex',out=180,in=90] node[below] {$0$} (s0in)
        (s1c) edge[-latex',out=-100,in=-90] node[below] {$0$} (s0in)
        (s1b) edge[-latex'] node[above] {$1$} (s2in)
  %      (s1c) edge[-latex'] node[above] {$1$} (s2in)
        %      (s0a) edge[-latex'] node[above,pos=0.75] {$1$} (sink)
        ;
      \end{scope}
      \end{scope}
    \end{tikzpicture}
    \caption{An implementation~$I_1$ and its product with the automaton of Fig.~\ref{fig-exaut}}\label{fig-eximpl}
  \end{figure*}
\end{example}
 
\subsection{Ommitted Proofs}
\begin{proof}[Proof of Lemma~\ref{lemma-complete-uniform}]
  Assume there exists a covering trace $\sigma$ in~$\calAR\otimes I$,
  and let~$K$ be the length of~$\sigma$.
  % that guides $\calAR \otimes I$ to some state $(o, s)$ for some state~$s$ of $I$.  
  When playing $\uniformtc_K$ \emph{ad~infinitum}, the~testing process
  restarts an infinite number of times.  At~each step, the~algorithm
  picks each valuation of the input variables with
  probability~$2^{-|\AP^\inp|}$.  So at each restart, the probability
  of choosing exactly $\sigma$ is $2^{-|\AP^\inp|\cdot|\sigma|}$.
  Therefore, $\sigma$~is picked eventually with probability~$1$.
\end{proof}

\begin{proof}[Proof of Lemma~\ref{lemma:mps}]
  Consider the computation of 
  \[\CPre_\calAR^*(B)= \lfp(\lambda X.(B \cup \CPre_\calAR(X))),\]
  and let $C_0,C_1,\ldots,C_n$ denote the iterates such that $C_0=\emptyset$,
  and $C_i = B \cup C_{i-1} \cup \CPre(C_{i-1})$ for~$i\geq 1$.
  For each state $s\in \CPre_\calAR^*(B)$, let~$i_s$ denote the least index such that
  $s \in C_i(B)$.
  We define $\st(s)=\{ \nu^\inp \mid \forall \nu^\out, 
  \Post_{\calAR}(s,{\vAP}^\inp\cup {\vAP}^\out) \in C_{i_s-1} \;\cup\; \error_\calAR \}$.
\end{proof}

\subsection{Non-Deterministic Implementations}
Although we restricted the presentation of the above algorithms to deterministic implementations, they all apply to
non-deterministic ones (that is, automata that are neither input-deterministic nor  output-deterministic).
%\todo{TJ: unclear to me whether we talk about non deterministic or non output-determinisitic implementations. For now we restrict to determinisitic, input and output-deterministic ones}
If the implementation is finite-state and purely randomized, that is, if it can be modelled by
a finite-state Markov chain, then all algorithms apply with the same completeness guarantees
since MCTS and similar algorithms have been designed for Markov decision processes.
On the other hand, if non-deterministic choices do not follow probability distributions, or cannot be modelled by a finite-state
Markov chain, then the algorithms do not have completeness guarantees. 
This is typically the case if the implementation is initialized in arbitrary state (e.g., due to states of the caches) and
possible initializations do not follow any particular probability distribution.
As usual in reinforcement learning, MCTS can still be applied and might converge or give useful results even though no theoretical
guarantees can be proven. Thus, all our algorithms can be applied in practice to non-deterministic implementations.
We~leave an empirical evaluation in such cases for future work.

\bibliographystyle{IEEEtran}
\bibliography{biblio}

\begin{thebibliography}{10}
\providecommand{\url}[1]{#1}
\csname url@samestyle\endcsname
\providecommand{\newblock}{\relax}
\providecommand{\bibinfo}[2]{#2}
\providecommand{\BIBentrySTDinterwordspacing}{\spaceskip=0pt\relax}
\providecommand{\BIBentryALTinterwordstretchfactor}{4}
\providecommand{\BIBentryALTinterwordspacing}{\spaceskip=\fontdimen2\font plus
\BIBentryALTinterwordstretchfactor\fontdimen3\font minus \fontdimen4\font\relax}
\providecommand{\BIBforeignlanguage}[2]{{%
\expandafter\ifx\csname l@#1\endcsname\relax
\typeout{** WARNING: IEEEtran.bst: No hyphenation pattern has been}%
\typeout{** loaded for the language `#1'. Using the pattern for}%
\typeout{** the default language instead.}%
\else
\language=\csname l@#1\endcsname
\fi
#2}}
\providecommand{\BIBdecl}{\relax}
\BIBdecl

\bibitem{blackburn1996t}
M.~R. Blackburn and R.~D. Busser, ``{T-VEC}: A tool for developing critical systems,'' in \emph{Proceedings of 11th Annual Conference on Computer Assurance. COMPASS'96}.\hskip 1em plus 0.5em minus 0.4em\relax IEEE, 1996, pp. 237--249.

\bibitem{tretmans2019model}
J.~Tretmans and P.~van~de Laar, ``Model-based testing with {T}or{X}akis,'' in \emph{Central European Conference on Information and Intelligent Systems}.\hskip 1em plus 0.5em minus 0.4em\relax Faculty of Organization and Informatics Varazdin, 2019, pp. 247--258.

\bibitem{jeannet2016debugging}
B.~Jeannet and F.~Gaucher, ``Debugging embedded systems requirements with {STIMULUS}: an automotive case-study,'' in \emph{8th European Congress on Embedded Real Time Software and Systems (ERTS 2016)}, 2016.

\bibitem{utting2012taxonomy}
M.~Utting, A.~Pretschner, and B.~Legeard, ``A taxonomy of model-based testing approaches,'' \emph{Software testing, verification and reliability}, vol.~22, no.~5, pp. 297--312, 2012.

\bibitem{li2018survey}
W.~Li, F.~Le~Gall, and N.~Spaseski, ``A survey on model-based testing tools for test case generation,'' in \emph{Tools and Methods of Program Analysis: 4th International Conference, TMPA 2017, Moscow, Russia, March 3-4, 2017, Revised Selected Papers 4}.\hskip 1em plus 0.5em minus 0.4em\relax Springer, 2018, pp. 77--89.

\bibitem{bolton2015programmable}
W.~Bolton, \emph{Programmable logic controllers}.\hskip 1em plus 0.5em minus 0.4em\relax Newnes, 2015.

\bibitem{beizer1995black}
B.~Beizer, \emph{Black-box testing: techniques for functional testing of software and systems}.\hskip 1em plus 0.5em minus 0.4em\relax John Wiley \& Sons, Inc., 1995.

\bibitem{tretmans2003torx}
G.~Tretmans and H.~Brinksma, ``Torx: Automated model-based testing,'' in \emph{First European Conference on Model-Driven Software Engineering}, 2003, pp. 31--43.

\bibitem{torxakis}
``Torxakis,'' \url{https://github.com/torxakis}.

\bibitem{Tretmans96-SCT}
J.~Tretmans, ``Test generation with inputs, outputs and repetitive quiescence,'' \emph{Softw. Concepts Tools}, vol.~17, no.~3, pp. 103--120, 1996.

\bibitem{sutton2018reinforcement}
R.~S. Sutton and A.~G. Barto, \emph{Reinforcement learning: An introduction}.\hskip 1em plus 0.5em minus 0.4em\relax MIT press, 2018.

\bibitem{Yannakakis2004}
M.~Yannakakis, ``Testing, optimization, and games,'' in \emph{Proceedings of the 19th Annual IEEE Symposium on Logic in Computer Science, 2004.}, 2004, pp. 78--88.

\bibitem{VRC-fates2006}
M.~Veanes, P.~Roy, and C.~Campbell, ``Online testing with reinforcement learning,'' in \emph{Formal Approaches to Software Testing and Runtime Verification}, K.~Havelund, M.~N{\'u}{\~{n}}ez, G.~Ro{\c{s}}u, and B.~Wolff, Eds.\hskip 1em plus 0.5em minus 0.4em\relax Berlin, Heidelberg: Springer Berlin Heidelberg, 2006, pp. 240--253.

\bibitem{ng1999policy}
A.~Y. Ng, D.~Harada, and S.~Russell, ``Policy invariance under reward transformations: Theory and application to reward shaping,'' in \emph{Icml}, vol.~99.\hskip 1em plus 0.5em minus 0.4em\relax Citeseer, 1999, pp. 278--287.

\bibitem{DBLP:journals/stvr/KorogluS21}
\BIBentryALTinterwordspacing
Y.~K{\"{o}}roglu and A.~Sen, ``Functional test generation from {UI} test scenarios using reinforcement learning for android applications,'' \emph{Softw. Test. Verification Reliab.}, vol.~31, no.~3, 2021. [Online]. Available: \url{https://doi.org/10.1002/stvr.1752}
\BIBentrySTDinterwordspacing

\bibitem{BEMN-oopsla24}
\BIBentryALTinterwordspacing
A.~Borgarelli, C.~Enea, R.~Majumdar, and S.~Nagendra, ``Reward augmentation in reinforcement learning for testing distributed systems,'' \emph{Proc. ACM Program. Lang.}, vol.~8, no. OOPSLA2, Oct. 2024. [Online]. Available: \url{https://doi.org/10.1145/3689779}
\BIBentrySTDinterwordspacing

\bibitem{coulom2006efficient}
R.~Coulom, ``Efficient selectivity and backup operators in monte-carlo tree search,'' in \emph{International conference on computers and games}.\hskip 1em plus 0.5em minus 0.4em\relax Springer, 2006, pp. 72--83.

\bibitem{browne2012mcts}
C.~B. Browne, E.~Powley, D.~Whitehouse, S.~M. Lucas, P.~I. Cowling, P.~Rohlfshagen, S.~Tavener, D.~Perez, S.~Samothrakis, and S.~Colton, ``A survey of monte carlo tree search methods,'' \emph{IEEE Transactions on Computational Intelligence and AI in Games}, vol.~4, no.~1, pp. 1--43, 2012.

\bibitem{fijalkow2025games}
\BIBentryALTinterwordspacing
N.~Fijalkow, C.~Aiswarya, G.~Avni, N.~Bertrand, P.~Bouyer, R.~Brenguier, A.~Carayol, A.~Casares, J.~Fearnley, H.~Gimbert, T.~A. Henzinger, P.~Gastin, F.~Horn, R.~Ibsen{-}Jensen, N.~Markey, B.~Monmege, P.~Novotn{\'{y}}, P.~Ohlmann, M.~Randour, O.~Sankur, S.~Schmitz, O.~Serre, M.~Skomra, N.~Sznajder, and P.~Vandenhove, \emph{Games on Graphs: From Logic and Automata to Algorithms}.\hskip 1em plus 0.5em minus 0.4em\relax Cambridge University Press, 2025, vol. abs/2305.10546, in press, Cambridge University Press. [Online]. Available: \url{https://doi.org/10.48550/arXiv.2305.10546}
\BIBentrySTDinterwordspacing

\bibitem{jard2005tgv}
C.~Jard and T.~J{\'e}ron, ``{TGV}: theory, principles and algorithms: A tool for the automatic synthesis of conformance test cases for non-deterministic reactive systems,'' \emph{International Journal on Software Tools for Technology Transfer}, vol.~7, pp. 297--315, 2005.

\bibitem{kocsis2006bandit}
L.~Kocsis and C.~Szepesv{\'a}ri, ``Bandit based {M}onte-{C}arlo planning,'' in \emph{Machine Learning: ECML 2006: 17th European Conference on Machine Learning Berlin, Germany, September 18-22, 2006 Proceedings 17}.\hskip 1em plus 0.5em minus 0.4em\relax Springer, 2006, pp. 282--293.

\bibitem{camacho2019ltl}
A.~Camacho, R.~T. Icarte, T.~Q. Klassen, R.~A. Valenzano, and S.~A. McIlraith, ``{LTL} and beyond: Formal languages for reward function specification in reinforcement learning.'' in \emph{IJCAI}, vol.~19, 2019, pp. 6065--6073.

\bibitem{DLLN08}
A.~David, K.~G. Larsen, S.~Li, and B.~Nielsen, ``A game-theoretic approach to real-time system testing,'' in \emph{{P}roceedings of the 2000 {D}esign, {A}utomation and {T}est in {E}urope ({DATE}'00)}.\hskip 1em plus 0.5em minus 0.4em\relax IEEE Comp. Soc. Press, Mar. 2008, pp. 486--491.

\bibitem{gelly2011monte}
S.~Gelly and D.~Silver, ``Monte-carlo tree search and rapid action value estimation in computer go,'' \emph{Artificial Intelligence}, vol. 175, no.~11, pp. 1856--1875, 2011.

\bibitem{BrenguierPRS14}
\BIBentryALTinterwordspacing
R.~Brenguier, G.~A. P{\'{e}}rez, J.~Raskin, and O.~Sankur, ``Abssynthe: abstract synthesis from succinct safety specifications,'' in \emph{Proceedings 3rd Workshop on Synthesis, {SYNT} 2014, Vienna, Austria, July 23-24, 2014}, ser. {EPTCS}, K.~Chatterjee, R.~Ehlers, and S.~Jha, Eds., vol. 157, 2014, pp. 100--116. [Online]. Available: \url{https://doi.org/10.4204/EPTCS.157.11}
\BIBentrySTDinterwordspacing

\bibitem{maroti2004flooding}
M.~Mar{\'o}ti, B.~Kusy, G.~Simon, and A.~L{\'e}deczi, ``The flooding time synchronization protocol,'' in \emph{Proceedings of the 2nd international conference on Embedded networked sensor systems}, 2004, pp. 39--49.

\bibitem{sankur2017abstraction}
O.~Sankur and J.-P. Talpin, ``An abstraction technique for parameterized model checking of leader election protocols: Application to ftsp,'' in \emph{International Conference on Tools and Algorithms for the Construction and Analysis of Systems}.\hskip 1em plus 0.5em minus 0.4em\relax Springer, 2017, pp. 23--40.

\bibitem{chang1979improved}
E.~Chang and R.~Roberts, ``An improved algorithm for decentralized extrema-finding in circular configurations of processes,'' \emph{Communications of the ACM}, vol.~22, no.~5, pp. 281--283, 1979.

\bibitem{HJM18}
\BIBentryALTinterwordspacing
L.~Henry, T.~J{\'{e}}ron, and N.~Markey, ``Control strategies for off-line testing of timed systems,'' in \emph{Model Checking Software - 25th International Symposium, {SPIN} 2018, Malaga, Spain, June 20-22, 2018, Proceedings}, ser. Lecture Notes in Computer Science, M.~Gallardo and P.~Merino, Eds., vol. 10869.\hskip 1em plus 0.5em minus 0.4em\relax Springer, 2018, pp. 171--189. [Online]. Available: \url{https://doi.org/10.1007/978-3-319-94111-0\_10}
\BIBentrySTDinterwordspacing

\bibitem{Ramangalahy98}
\BIBentryALTinterwordspacing
S.~Ramangalahy, ``Strategies for conformance testing,'' Max Planck Institut F\"{u}r Informatik, Tech. Rep. MPI-I-98-010, May 1998. [Online]. Available: \url{https://hdl.handle.net/21.11116/0000-0001-7559-8}
\BIBentrySTDinterwordspacing

\bibitem{DLLN-MBT08}
A.~David, K.~Larsen, S.~Li, and B.~Nielsen, ``Cooperative testing of timed systems,'' in \emph{4th Workshop on Model Based Testing (MBT'08)}, vol. 220, 12 2008, pp. 79--92.

\bibitem{Veanes2008}
\BIBentryALTinterwordspacing
M.~Veanes, C.~Campbell, W.~Grieskamp, W.~Schulte, N.~Tillmann, and L.~Nachmanson, ``Model-based testing of object-oriented reactive systems with {S}pec {E}xplorer,'' in \emph{Formal Methods and Testing: An Outcome of the FORTEST Network, Revised Selected Papers}, R.~M. Hierons, J.~P. Bowen, and M.~Harman, Eds.\hskip 1em plus 0.5em minus 0.4em\relax Berlin, Heidelberg: Springer Berlin Heidelberg, 2008, pp. 39--76. [Online]. Available: \url{https://doi.org/10.1007/978-3-540-78917-8_2}
\BIBentrySTDinterwordspacing

\bibitem{MMS-tacas18}
L.~Marsso, R.~Mateescu, and W.~Serwe, ``Testor: A modular tool for on-the-fly conformance test case generation,'' in \emph{Tools and Algorithms for the Construction and Analysis of Systems}, D.~Beyer and M.~Huisman, Eds.\hskip 1em plus 0.5em minus 0.4em\relax Cham: Springer International Publishing, 2018, pp. 211--228.

\bibitem{belinfante2010jtorx}
A.~Belinfante, ``Jtorx: A tool for on-line model-driven test derivation and execution,'' in \emph{International Conference on Tools and Algorithms for the Construction and Analysis of Systems}.\hskip 1em plus 0.5em minus 0.4em\relax Springer, 2010, pp. 266--270.

\bibitem{hessel2008testing}
A.~Hessel, K.~G. Larsen, M.~Mikucionis, B.~Nielsen, P.~Pettersson, and A.~Skou, ``Testing real-time systems using uppaal,'' \emph{Formal Methods and Testing: An Outcome of the FORTEST Network, Revised Selected Papers}, pp. 77--117, 2008.

\bibitem{MPRS-icse11}
\BIBentryALTinterwordspacing
L.~Mariani, M.~Pezz\`{e}, O.~Riganelli, and M.~Santoro, ``Autoblacktest: A tool for automatic black-box testing,'' in \emph{Proceedings of the 33rd International Conference on Software Engineering}, ser. ICSE '11.\hskip 1em plus 0.5em minus 0.4em\relax New York, NY, USA: Association for Computing Machinery, 2011, p. 1013–1015. [Online]. Available: \url{https://doi.org/10.1145/1985793.1985979}
\BIBentrySTDinterwordspacing

\bibitem{adamo2018reinforcement}
D.~Adamo, M.~K. Khan, S.~Koppula, and R.~Bryce, ``Reinforcement learning for android {GUI} testing,'' in \emph{Proceedings of the 9th ACM SIGSOFT International Workshop on Automating Test Case Design, Selection, and Evaluation}, 2018, pp. 2--8.

\bibitem{LPSLY-ase22}
\BIBentryALTinterwordspacing
Z.~Lv, C.~Peng, Z.~Zhang, T.~Su, K.~Liu, and P.~Yang, ``Fastbot2: Reusable automated model-based {GUI} testing for android enhanced by reinforcement learning,'' in \emph{Proceedings of the 37th IEEE/ACM International Conference on Automated Software Engineering}, ser. ASE '22.\hskip 1em plus 0.5em minus 0.4em\relax New York, NY, USA: Association for Computing Machinery, 2023. [Online]. Available: \url{https://doi.org/10.1145/3551349.3559505}
\BIBentrySTDinterwordspacing

\bibitem{reddy2020quickly}
S.~Reddy, C.~Lemieux, R.~Padhye, and K.~Sen, ``Quickly generating diverse valid test inputs with reinforcement learning,'' in \emph{Proceedings of the ACM/IEEE 42nd International Conference on Software Engineering}, 2020, pp. 1410--1421.

\bibitem{THMT-ase2021}
U.~C. Türker, R.~M. Hierons, M.~R. Mousavi, and I.~Y. Tyukin, ``Efficient state synchronisation in model-based testing through reinforcement learning,'' in \emph{2021 36th IEEE/ACM International Conference on Automated Software Engineering (ASE)}, 2021, pp. 368--380.

\bibitem{PZAdSL-date2020}
N.~Pfeifer, B.~V. Zimpel, G.~A.~G. Andrade, and L.~C.~V. dos Santos, ``A reinforcement learning approach to directed test generation for shared memory verification,'' in \emph{2020 Design, Automation \& Test in Europe Conference \& Exhibition (DATE)}, 2020, pp. 538--543.

\bibitem{RMCT-acm2021}
\BIBentryALTinterwordspacing
A.~Romdhana, A.~Merlo, M.~Ceccato, and P.~Tonella, ``Deep reinforcement learning for black-box testing of android apps,'' \emph{ACM Trans. Softw. Eng. Methodol.}, 2021. [Online]. Available: \url{https://doi.org/10.1145/3502868}
\BIBentrySTDinterwordspacing

\bibitem{FCP-mcs23}
R.~Fogler, I.~Cohen, and D.~Peled, ``Accelerating black box testing with light-weight learning,'' in \emph{Model Checking Software}, G.~Caltais and C.~Schilling, Eds.\hskip 1em plus 0.5em minus 0.4em\relax Cham: Springer Nature Switzerland, 2023, pp. 103--120.

\bibitem{velasquez2021dynamic}
\BIBentryALTinterwordspacing
A.~Velasquez, B.~Bissey, L.~Barak, A.~Beckus, I.~Alkhouri, D.~Melcer, and G.~Atia, ``Dynamic automaton-guided reward shaping for {M}onte {C}arlo tree search,'' \emph{Proceedings of the AAAI Conference on Artificial Intelligence}, vol.~35, no.~13, pp. 12\,015--12\,023, 2021. [Online]. Available: \url{https://ojs.aaai.org/index.php/AAAI/article/view/17427}
\BIBentrySTDinterwordspacing

\bibitem{giannakopoulou2020formal}
D.~Giannakopoulou, A.~Mavridou, J.~Rhein, T.~Pressburger, J.~Schumann, and N.~Shi, ``Formal requirements elicitation with fret,'' in \emph{International Working Conference on Requirements Engineering: Foundation for Software Quality (REFSQ-2020)}, no. ARC-E-DAA-TN77785, 2020.

\bibitem{BeckerDHHLPW21}
S.~Becker, D.~Dietsch, N.~Hauff, E.~Henkel, V.~Langenfeld, A.~Podelski, and B.~Westphal, ``Hanfor: Semantic requirements review at scale,'' in \emph{{REFSQ} Workshops}, ser. {CEUR} Workshop Proceedings, vol. 2857.\hskip 1em plus 0.5em minus 0.4em\relax CEUR-WS.org, 2021.

\bibitem{SNJM-reqkit25}
O.~Sankur, R.~Noguchi, T.~J{\'e}ron, and D.~Mentr{\'e}, ``Reqkit: A real-time requirement analysis tool kit,'' Tech. Rep., 2025.

\bibitem{BGRSS-nfm25}
N.~Bertrand, P.~Ghorpade, S.~Rubin, B.~Scholz, and P.~Suboti{\'{c}}, ``Reusable formal verification of dag-based consensus protocols,'' in \emph{NASA Formal Methods}, A.~Dutle, L.~Humphrey, and L.~Titolo, Eds.\hskip 1em plus 0.5em minus 0.4em\relax Cham: Springer Nature Switzerland, 2025, pp. 138--158.

\bibitem{DHS-vldb20}
\BIBentryALTinterwordspacing
A.~J.~J. Davis, M.~Hirschhorn, and J.~Schvimer, ``Extreme modelling in practice,'' \emph{Proc. VLDB Endow.}, vol.~13, no.~9, p. 1346–1358, May 2020. [Online]. Available: \url{https://doi.org/10.14778/3397230.3397233}
\BIBentrySTDinterwordspacing

\bibitem{PLC-mitsu}
Mitsubishi{ }Electric{ }Corporation, ``Mitsubishi programmable controller -- {T}raining manual,'' 2012, \url{https://dl.mitsubishielectric.com/dl/fa/document/manual/school\_text/sh081123eng/sh081123enga.pdf}.

\end{thebibliography}

\end{document}